\documentclass[sigconf]{acmart}
\usepackage{multirow}
\usepackage{amsmath}
\usepackage{array}
\usepackage[table]{xcolor}
\usepackage{siunitx}
\definecolor{NYCband}{RGB}{248,252,255}
\definecolor{TKYband}{RGB}{250,248,255}
\definecolor{CAband}{RGB}{250,255,248}
\definecolor{FullRow}{RGB}{245,245,245}
\definecolor{BestCell}{RGB}{255,235,210}
\newcolumntype{Y}{>{\columncolor{NYCband}}S[table-format=1.4]}
\newcolumntype{T}{>{\columncolor{TKYband}}S[table-format=1.4]}
\newcolumntype{C}{>{\columncolor{CAband}}S[table-format=1.4]}
\newcommand{\bestab}[1]{\multicolumn{1}{>{\columncolor{BestCell}\bfseries}S[table-format=1.4]}{#1}}
\newcommand{\rmv}{\textcolor{black!55}{\scriptsize($-$)}}
\definecolor{Top2Gray}{RGB}{90,90,90}
\definecolor{ImprovNum}{RGB}{220,20,120}
\newcommand{\best}[1]{{\bfseries\color{black}#1}}
\newcommand{\second}[1]{{\color{Top2Gray}\underline{#1}}}
\newcommand{\up}[1]{{\color{ImprovNum}\bfseries\small$\uparrow$\,#1}}
\usepackage{booktabs}
\usepackage{graphicx}
\usepackage[toc]{appendix}
\usepackage{subcaption}
\newcommand{\hNYC}[1]{\multicolumn{1}{>{\columncolor{NYCband}}c}{#1}}
\newcommand{\hTKY}[1]{\multicolumn{1}{>{\columncolor{TKYband}}c}{#1}}
\newcommand{\hCA}[1]{\multicolumn{1}{>{\columncolor{CAband}}c}{#1}}
\AtBeginDocument{%
 }

\setcopyright{acmlicensed}
\copyrightyear{2018}
\acmYear{2018}
\acmDOI{XXXXXXX.XXXXXXX}
\acmConference[Conference acronym 'XX]{Make sure to enter the correct
  conference title from your rights confirmation email}{June 03--05,
  2018}{Woodstock, NY}
\acmISBN{978-1-4503-XXXX-X/2018/06}

\begin{document}

\title{Mag-Mamba: Modeling Coupled spatiotemporal Asymmetry for POI Recommendation}
\author{Zhuoxuan Li}
\authornote{Both authors contributed equally to this research.}
\affiliation{
  \institution{College of Computer Science and Technology, Tongji University}
  \city{Shanghai}
  \country{China}
}
\email{li\_zhuoxuan@outlook.com}

\author{Tangwei Ye}
\authornotemark[1]
\affiliation{
  \institution{College of Computer Science and Technology, Tongji University}
  \city{Shanghai}
  \country{China}
}
\email{yetw@tongji.edu.cn}

\author{Jieyuan Pei}
\affiliation{
  \institution{College of Information Engineering, Zhejiang University of Technology}
  \city{Hangzhou}
  \country{China}
}
\email{peijieyuan@zjut.edu.cn}

\author{Haina Liang}
\affiliation{
  \institution{School of Mathematics Sciences, Tongji University}
  \city{Shanghai}
  \country{China}
}
\email{RationalHaina@tongji.edu.cn}

\author{Zhongyuan Lai}
\affiliation{
  \institution{Shanghai Ballsnow Intelligent Technology Co., Ltd}
  \city{Shanghai}
  \country{China}
}
\email{zhongyuan.lai@ballsnow.com}

\author{Zihan Liu}
\affiliation{
  \institution{College of Computer Science and Technology, Tongji University}
  \city{Shanghai}
  \country{China}
}
\email{tongjilzh@gmail.com}

\author{Yiming Wu}
\affiliation{
  \institution{College of Computer Science and Technology, Tongji University}
  \city{Shanghai}
  \country{China}
}
\email{1875708412@qq.com}

\author{Qi Zhang}
\affiliation{
  \institution{College of Computer Science and Technology, Tongji University}
  \city{Shanghai}
  \country{China}
}
\email{zhangqi\_cs@tongji.edu.cn}

\author{Liang Hu}
\authornote{Corresponding author}
\affiliation{
  \institution{College of Computer Science and Technology, Tongji University}
  \city{Shanghai}
  \country{China}
}
\email{rainmilk@gmail.com}

\renewcommand{\shortauthors}{Li et al.}

\begin{abstract}
Next Point-of-Interest (POI) recommendation is a critical task in location-based services, yet it faces the fundamental challenge of coupled spatiotemporal asymmetry inherent in urban mobility. Specifically, transition intents between locations exhibit high asymmetry and are dynamically conditioned on time. Existing methods, typically built on graph or sequence backbones, rely on symmetric operators or real-valued aggregations, struggling to unify the modeling of time-varying global directionality. To address this limitation, we propose Mag-Mamba, a framework whose core insight lies in modeling spatiotemporal asymmetry as phase-driven rotational dynamics in the complex domain. Based on this, we first devise a Time-conditioned Magnetic Phase Encoder that constructs a time-conditioned Magnetic Laplacian on the geographic adjacency graph, utilizing edge phase differences to characterize the globally evolving spatial directionality. Subsequently, we introduce a Complex-valued Mamba module that generalizes traditional scalar state decay into joint decay-rotation dynamics, explicitly modulated by both time intervals and magnetic geographic priors. Extensive experiments on three real-world datasets demonstrate that Mag-Mamba achieves significant performance improvements over state-of-the-art baselines.
\end{abstract}


\received{20 February 2007}
\received[revised]{12 March 2009}
\received[accepted]{5 June 2009}
\maketitle
\section{Introduction}
With the rapid proliferation of Location-Based Social Networks (LBSNs) and mobile internet platforms, massive amounts of geo-tagged user behavior data have continuously accumulated \cite{sanchez2022point,wang2019sequential,yang2014modeling}. Consequently, predicting the next Point-of-Interest (POI) based on historical trajectories—known as POI recommendation—has become a critical capability for location-based services, intelligent navigation, and local lifestyle platforms. Unlike traditional sequence recommendation, urban mobility exhibits significant spatiotemporal asymmetry. The accessibility, traffic flow, and semantic intent between the same pair of locations often differ in opposite directions, and these differences vary with temporal conditions. This requires POI recommendation not only to answer "where to go next" but also to handle complex spatiotemporal dependencies within a dynamically evolving anisotropic field.

\begin{figure}
\captionsetup{skip=3pt,belowskip=-15pt}
\centering
\includegraphics[width=\linewidth]{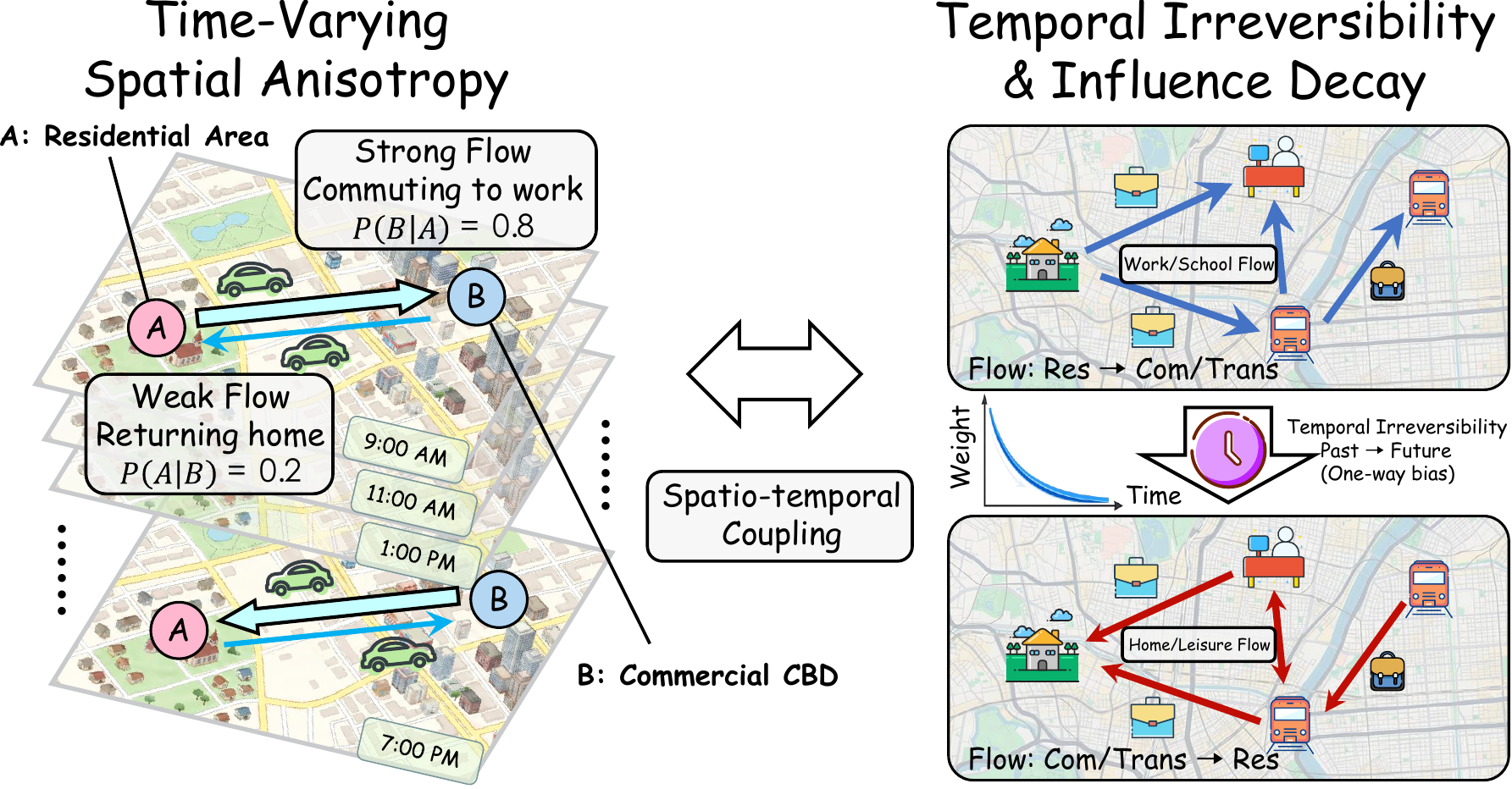}
\caption{Illustration of the Coupled Spatiotemporal Asymmetry in urban mobility.}
\label{task}
\end{figure}

This asymmetry manifests in real-world scenarios as a deep coupling between two dimensions. As illustrated in Figure \ref{task}, first, the dominant flow direction in urban topology is highly time-conditioned. Influenced by one-way streets, functional zoning (e.g., flow from residential to commercial areas), and traffic patterns, the transition intensity from location $A$ to $B$ is often significantly different from that of $B$ to $A$ (i.e., $P(B|A) \neq P(A|B)$) under a given temporal context. Moreover, the dominant direction may shift or even reverse over time. Second, time intervals not only determine the decay strength of historical information but also modulate the magnitude of spatial constraints during state recurrence. Consequently, the model is required to simultaneously learn the time-conditioned direction field and how time intervals scale the intensity of sequence updates.

Despite significant progress in existing recommendation methods, the characterization of this coupled asymmetry remains insufficient. Graph-based methods that leverage geo-aware GNN encoders  \cite{li2021you,qin2023disenpoi,xu2023revisiting} can encode high-order geographic proximity but typically rely on symmetric operators in the real domain (such as undirected adjacency matrices) or decouple directed relationships into separate in-degree and out-degree structures, making it difficult to stably characterize reciprocal directions and their temporal switching under a unified parameterization. Sequential methods based on LSTM and Transformer \cite{baral2018caps,baral2018close,wang2019spent} predominantly rely on real-valued similarity aggregation, where directional constraints are usually injected as external features. Meanwhile, recent POI recommenders using state-space and Mamba-like architectures \cite{chen2024geomamba,jiang2026hierarchical,qin2025geomamba} offer linear-time long-sequence modeling, but most designs still operate with real-valued state compression and updates, lacking an explicit mechanism to encode geographic directionality and a time-varying direction field. The goal of this paper is to incorporate the temporal switching of the direction field and the update scaling of time intervals into a unified recurrence mechanism.

To directly address this coupling challenge, we adopt a phase rotation perspective in the complex domain to model spatiotemporal asymmetry. In the complex plane, a unit complex number $e^{i\theta}$ corresponds to a rotation by angle $\theta$, and a reciprocal relationship naturally corresponds to phase negation. This allows the model to characterize reciprocal directions under a unified parameterization, avoiding the need to learn two independent representations for opposite directions. This concept has been proven effective for expressing asymmetric relationships in knowledge graph reasoning \cite{sun2019rotate} and magnetic Laplacian modeling on directed graphs \cite{zhang2021magnet}. However, treating phases merely as local free parameters makes it difficult to guarantee the global consistency and decomposability of phase signals across the entire graph. To this end, we introduce the Magnetic Laplacian as a structured generation mechanism for phases. Leveraging its stable spectral structure, we can extract globally optimized directional bases from its eigenvector phases. Finally, to characterize "how time intervals scale updates," we map time intervals to continuous step sizes and utilize the directional signals generated by magnetic phases as rotation drivers to jointly perform sequence updates.

Based on this, we propose Mag-Mamba. The model first constructs a time-dependent directional distribution on the geographic adjacency graph derived from the directional transition statistics of real trajectories via a magnetic phase encoder, forming a corresponding Magnetic Laplacian to generate edge phase differences. Subsequently, these phase differences are injected into a Complex-valued Mag-Mamba layer as rotation drivers. We generalize the real-valued update of traditional SSMs to a decay-plus-rotation dynamics, where the phase differences serve as explicit rotation signals and continuous step sizes characterize the scaling of update intensity by time intervals, thereby achieving a unified modeling of the coupled spatiotemporal asymmetry.

Our contributions are summarized as follows:
\begin{itemize}
    \item We propose a Time-conditioned Magnetic Phase Encoder that explicitly models and dynamically characterizes time-varying asymmetric geographic flows in the complex plane through low-rank temporal basis decomposition.
    \item We propose a Complex-valued Mamba module that generalizes scalar decay to decay-rotation dynamics, realizing spatiotemporal sequence modeling jointly modulated by time intervals and magnetic phase priors.
    \item Extensive experiments on three real-world LBSN datasets demonstrate that Mag-Mamba achieves superior overall performance compared to state-of-the-art baselines.
\end{itemize}

\section{Related Work}
\label{Related_Work}
\subsection{Next POI Recommendation}
Next Point-of-Interest (POI) recommendation aims to explore transition patterns between interest points based on user check-in trajectories to predict the next visited location. Early research primarily formulated next location prediction as a sequence modeling problem, utilizing models such as LSTMs or RNNs \cite{feng2018deepmove,liu2016predicting,wang2021reinforced,wu2020personalized}. With the advancement of attention mechanisms, Transformer-based models were widely introduced, as the global re-weighting of historical check-ins enhanced the capability to model non-adjacent and discontinuous visits \cite{duan2023clsprec,luo2021stan,xue2021mobtcast,zhang2022next,xie2023hierarchical}. Representative works further injected geographical inductive biases into attention mechanisms to better incorporate spatial constraints. Concurrently, Graph Neural Networks (GNNs) capture high-order movement patterns beyond individual trajectories by constructing POI transition graphs \cite{li2021you,qin2023disenpoi,wang2022learning,wang2023adaptive,xu2023revisiting,yan2023spatio,wang2022modeling}. 

Recently, Mamba, based on Selective State Space Models (SSMs), has emerged as a new paradigm for sequence modeling due to its linear complexity and hardware-friendly parallelization capabilities. Chen et al. combined hierarchical geo-coding with Mamba to achieve geo-sequence awareness \cite{chen2024geomamba}; Qin et al. extended the GaPPO operator to the state space to model multi-granular spatiotemporal transitions \cite{qin2025geomamba}; Li et al. proposed a Mamba updating mechanism within the tangent space of hyperbolic geometry \cite{li2025gtr}; and TCFMamba combined Mamba with trajectory collaborative filtering to address data bias \cite{qian2025tcfmamba}. However, most of these methods model spatial directionality merely as auxiliary features and treat time as static embeddings, lacking a unified mechanism to explicitly model the asymmetry of spatiotemporal coupling.

\subsection{Directed Graph Modeling}
Real-world movement trajectories inherently form non-reciprocal directed graphs, presenting a challenge for traditional undirected spectral operators. Classical spectral graph theory laid the groundwork by establishing spectral tools for irreversible Markov chains \cite{chung2005laplacians}. Building on this, Hermitian adjacency matrices and Magnetic Laplacian operators encode undirected geometric structure in the magnitude, thereby enabling a principled spectral analysis of directed graphs \cite{guo2017hermitian, fanuel2018magnetic}. MagNet \cite{zhang2021magnet} instantiates this concept into a directed Graph Neural Network (GNN), where the phase explicitly represents directionality through complex Hermitian matrices derived from the Magnetic Laplacian. These directed spectral operators are well-suited for generating structure-aware geographical location encodings on directed graphs. Recent works have also explored alternative Hermitian constructions to achieve better denoising performance \cite{badea2025haar, pahng2024improving}. However, existing directed spectral methods are primarily developed for static directed graphs and rarely involve direction fields under temporal conditions. This work fills this gap by mapping time-conditioned geographical direction fields to magnetic phases.

\subsection{Complex-Valued Rotations}
Complex-valued representations have been repeatedly proven effective in modeling asymmetric relationships. In Knowledge Graph representation learning, ComplEx \cite{trouillon2016complex} uses complex bilinear scoring to model antisymmetric and non-commutative relational patterns, while RotatE \cite{sun2019rotate} explicitly represents relations as rotations to map head entities to tail entities, providing a simple yet expressive mechanism for asymmetric relations. With the development of representation learning, rotation mechanisms have extended from the complex plane to higher-dimensional spaces and non-Euclidean geometries \cite{zhang2019quaternion, feng2022role}. In POI recommendation, ROTAN \cite{feng2024rotan} employs a rotation-based temporal attention mechanism to capture asymmetric temporal influences across different time slots. However, it overlooks the interplay between geographic directionality and temporal dynamics. To address this limitation, we propose to unify the modeling of the coupling relationship between spatial directionality and temporal dynamics within the complex domain.

\section{Preliminaries}
\subsection{Problem Definition}
Let $\mathcal{U}$, $\mathcal{P}$, and $\mathcal{C}$ denote the user set, POI set, and category set, respectively. The number of POIs is $|\mathcal{P}|$, the number of users is $|\mathcal{U}|$, and the number of categories is $|\mathcal{C}|$. Each POI $p\in\mathcal{P}$ is associated with a category label $c_p\in\mathcal{C}$. A check-in record is defined as $s=(u,p,\tau,c_p)$, which indicates that user $u\in\mathcal{U}$ visited POI $p\in\mathcal{P}$ at timestamp $\tau$, with category information $c_p$.
 We represent user $u$'s check-in sequence as $S^{u}=\langle s_{1},s_{2},\ldots,s_{L} \rangle$, where $s_L$ is the $L$-th check-in and $L$ is the sequence length. Given $S^u$, the goal of next-POI recommendation is to predict the next POI that user $u$ is most likely to visit, denoted as $p_{L+1}$.

\subsection{Spatiotemporal Asymmetry}
Urban mobility flows often exhibit pronounced time-conditioned directionality. For any POI pair $i,j\in\mathcal{P}$, the transition probability can be asymmetric under different temporal contexts: $P(j\mid i,b)\neq P(i\mid j,b)$, where $b\in\{0,1,\ldots,N_b-1\}$ denotes the discretized time bin. In particular, mobility can show tidal patterns, e.g., transitions $i\!\to\! j$ dominate in the morning, while $j\!\to\! i$ dominate in the afternoon. However, geographic adjacency is typically modeled by an undirected graph (static and symmetric), which encodes static reachability but leaves edge orientation unspecified. We keep a fixed undirected connectivity $\mathbf{W}$ to capture static reachability, and inject time-conditioned directionality via complex phases. The resulting phase differences serve as explicit control inputs that drive the sequence dynamics in Mag-Mamba.

\section{The Proposed Model}
\label{sec:m}
\begin{figure*}
\captionsetup{skip=2pt,belowskip=-5pt}
 \centering
 \includegraphics[width=\textwidth]{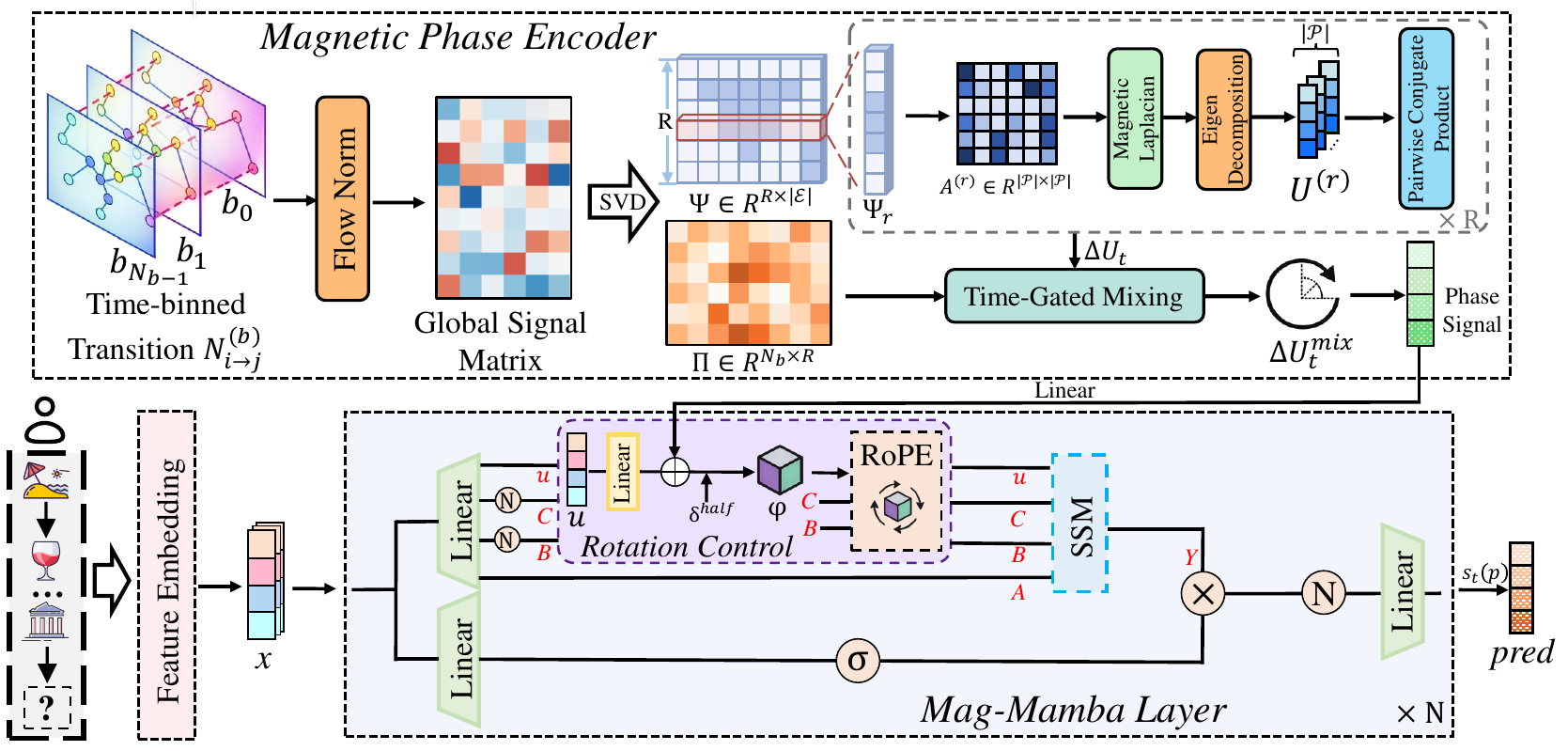}
\caption{Overview of the proposed Mag-Mamba framework, consisting of a Magnetic Phase Encoder, a context feature embedding module, and a Mag-Mamba layer.}
\label{Mag}
 \Description[none]{}
\end{figure*}
\subsection{Overview}
Figure~\ref{Mag} illustrates the proposed \textbf{Mag-Mamba} framework, consisting of three modules:
(i) a \textbf{Magnetic Phase Encoder} that reconstructs time-varying spectral directionality on the geographic graph via a low-rank spatiotemporal factorization;
(ii) a \textbf{context feature embedding} module that integrates user-specific preference representations and temporal rhythms;
and (iii) a \textbf{Mag-Mamba layer} that processes the resulting sequence with a state space model featuring complex-induced rotation dynamics, implemented as interleaved real 2D rotations whose angles are driven by graph-induced geometric phase differences. We detail these components next.

\subsection{Magnetic Phase Encoder}
This module transforms discrete trajectory statistics into time-conditioned phase signals while keeping the undirected geographic connectivity fixed. The phase signals will drive the rotation dynamics in Mag-Mamba.

\paragraph{Geographic graph construction.}
Following common practice in geographic encoding \cite{qin2023disenpoi}, we construct a radius-based adjacency graph $\mathcal{G}=(\mathcal{P},\mathcal{E})$ with POIs as nodes. An undirected edge $\{i,j\}\in\mathcal{E}$ is created if the geographic distance $d(i,j)$ is within a radius threshold $r$. We define symmetric geographic weights
\begin{equation}
\label{eq:geo_weight}
W_{ij}=\exp\!\left(-\frac{d(i,j)}{\sigma_{\mathrm{geo}}}\right),\quad
W_{ij}=W_{ji},\quad
W_{ii}=0,
\end{equation}
where $\sigma_{\mathrm{geo}}$ controls the decay scale and $W_{ij}=0$ if $\{i,j\}\notin\mathcal{E}$. This term captures static geographic reachability; directionality is entirely carried by the subsequent phase terms.

\paragraph{From transition counts to antisymmetric edge signals.}
We discretize time context into bins $b\in\{0,1,\ldots,N_b-1\}$. Let $N^{(b)}_{i\rightarrow j}$ denote the number of transitions from POI $i$ to POI $j$ observed in time bin $b$. We then quantify the directional asymmetry on each undirected edge within time bin $b$ by contrasting the two opposite transition counts. For each undirected edge $e=\{i,j\}\in\mathcal{E}$ (assume $i<j$), we construct a smoothed log-ratio and squash it into $[-1,1]$:
\begin{equation}
s^{(b)}_{e}
=
\tanh\!\left(
\frac{\log(N^{(b)}_{i\rightarrow j}+\alpha)-\log(N^{(b)}_{j\rightarrow i}+\alpha)}{\kappa}
\right)\in[-1,1],
\end{equation}
where $\alpha$ is a smoothing constant and $\kappa$ is a temperature parameter. The sign of $s^{(b)}_{e}$ indicates the dominant direction (positive: $i\!\to\! j$, negative: $j\!\to\! i$), and its magnitude measures directional strength. This construction yields opposite signs for the two directions on the same undirected edge, matching the Hermitian constraint used later. We then construct the global signal matrix $S\in\mathbb{R}^{N_b\times|\mathcal{E}|},\quad
S_{b,e}=s^{(b)}_{e}$ by stacking all bins and edges.

\paragraph{Low-rank spatiotemporal basis factorization.}
Maintaining an independent directional graph for every time bin is computationally prohibitive. We therefore compress the time-conditioned edge direction matrix
$S\in\mathbb{R}^{N_b\times|\mathcal{E}|}$ using a truncated SVD,
\begin{equation}
S = U\,\mathrm{diag}(\boldsymbol{\sigma})\,V^\top,
\end{equation}
where $U$ provides time-bin factors, $V^\top$ provides edge-wise spatial modes, and $\boldsymbol{\sigma}$ contains their energies. In implementation, we retain the top-$R$ components and fold the singular values into the time-bin factors, yielding two matrices $(\Pi,\Psi)$ such that
\begin{equation}
S \approx \Pi\Psi,\qquad
\Pi \equiv U\,\mathrm{diag}(\boldsymbol{\sigma}),\quad
\Psi \equiv V^\top,
\end{equation}
with $\Pi\in\mathbb{R}^{N_b\times R}$ as the (signed) time-bin mixing coefficients and $\Psi\in\mathbb{R}^{R\times|\mathcal{E}|}$ as the shared edge-direction bases. With this decomposition, time-varying directionality is represented as a signed linear mixture of a small set of shared spatial modes, enabling real-time phase feature construction while keeping all spectral computations as one-time precomputation. This discrete basis formulation is also statistically more stable under sparse check-ins than estimating a separate continuous-time directed graph.

\paragraph{Antisymmetric direction fields and Hermitian complex adjacency.}
For spectral stability, the magnetic Laplacian uses a Hermitian complex adjacency. For each basis $r$, we lift the edge coefficients $\Psi$ into an antisymmetric node-pair field $A^{(r)}\in\mathbb{R}^{|\mathcal{P}|\times |\mathcal{P}|}$. Let each undirected edge $e=\{u,v\} \in \mathcal{E}$ with $u<v$ be assigned a unique index $k$, and let $e_{ij}$ denote this index corresponding to the edge between $i$ and $j$.
\begin{equation}
A^{(r)}_{ij}=
\begin{cases}
\ \ \Psi_{r, e_{ij}}, & \{i,j\}\in\mathcal{E},\ i<j,\\
-\,\Psi_{r, e_{ij}}, & \{i,j\}\in\mathcal{E},\ i>j,\\
\ \ 0, & \text{otherwise}.
\end{cases}
\label{eq:main_A_def}
\end{equation}
This definition flips the sign when swapping $(i,j)$, hence $(A^{(r)})^\top=-A^{(r)}$.

This antisymmetry ensures $\Phi^{(r)}_{ji}=-\Phi^{(r)}_{ij}$ and yields a Hermitian complex adjacency, which is a necessary structural condition for the spectral decomposition to be well-posed and numerically stable.

We fix $q$ as a scalar magnetic charge (hyperparameter). Following MagNet \cite{zhang2021magnet}, we define the phase and the corresponding Hermitian adjacency as:
\begin{equation}
\Phi^{(r)}_{ij}=2\pi q\,A^{(r)}_{ij},\quad
H^{(r)}_{ij}=W_{ij}\exp(\mathrm{i}\Phi^{(r)}_{ij}),
\label{eq:Phi_H_r} 
\end{equation}
where $\mathrm{i}=\sqrt{-1}$. Since $\Psi=V^\top$ has orthonormal rows, $|A^{(r)}_{ij}|\le 1$, hence $|\Phi^{(r)}_{ij}|\le 2\pi q$. Let $D_g=\mathrm{diag}(d_1,\ldots,d_{|\mathcal{P}|})$ with $d_i=\sum_{j} W_{ij}$. The normalized magnetic Laplacian is
\begin{equation}
L^{(r)} = I_{|\mathcal{P}|} - D_g^{-1/2}H^{(r)}D_g^{-1/2}.
\label{eq:L_r}
\end{equation}
Since $H^{(r)}$ is Hermitian, $L^{(r)}$ is Hermitian positive semidefinite and admits an orthonormal eigen-decomposition. This construction guarantees that $L^{(r)}$ is Hermitian positive semidefinite (Lemma \ref{lem:app_L_psd}), ensuring a valid real spectrum for stable basis extraction. Let $V^{(r)}\in\mathbb{C}^{|\mathcal{P}|\times k}$ be the eigenvectors corresponding to the smallest $k$ eigenvalues. To remove instability from magnitude scaling and global phase ambiguity, we keep only elementwise phases and normalize them to unit-modulus phase tokens:
\begin{equation}
U^{(r)}(i,m)=\exp(\mathrm{i}\,\arg(V^{(r)}(i,m)))\in\mathbb{C},\quad
|U^{(r)}(i,m)|=1.
\label{eq:U_r} 
\end{equation}
This makes the embedding insensitive to numerical scaling and column-wise phase drift, improving stability of online phase difference computation.

\paragraph{Time-conditioned phase differences for each step.}
Denote the POI and timestamp in $s_t$ by $p_t$ and $\tau_t$, i.e., $s_t=(u,p_t,\tau_t,c_{p_t})$. We define the time gap as $\Delta t_t = \tau_t-\tau_{t-1}$ for $t \ge 2$ (with $\Delta t_1 = 0$).

We assign each step to a discrete time bin (hour-of-week). For a sequence step $t$ with current POI $p_t$, define the source POI
\begin{equation}
p^{\mathrm{src}}_t=
\begin{cases}
p_1, & t=1,\\
p_{t-1}, & t\ge 2.
\end{cases}
\end{equation}

Let $\mathbf{u}^{(r)}_p \in \mathbb{C}^k$ denote the row vector of $U^{(r)}$ corresponding to POI $p$. For each basis $r$, the phase-difference vector is
\begin{equation}
\Delta U^{(r)}_t=\mathbf{u}^{(r)}_{p_t} \odot \overline{\mathbf{u}^{(r)}_{p^{\mathrm{src}}_t}}
\in\mathbb{C}^{k},
\label{eq:DeltaU_r_t} 
\end{equation}
where $\odot$ is the elementwise product and $\overline{(\cdot)}$ denotes complex conjugation. Crucially, as proved in Proposition~\ref{prop:app_gauge_invariance}, the conjugate product form in Eq.~\ref{eq:DeltaU_r_t} makes each $\Delta U_t^{(r)}$ gauge-invariant by canceling the arbitrary column-wise phase ambiguity of eigenvectors. Since Eq.~\ref{eq:DeltaU_mix_t} is a linear combination of gauge-invariant vectors with real coefficients $\Pi_{b_t,r}$, $\Delta U_t^{\operatorname{mix}}$ remains gauge-invariant. Semantically, it encodes the geometric alignment between the user's movement and the underlying traffic flow, independent of the spectral coordinate system.

We mix bases using $\Pi$:
\begin{equation}
\Delta U_t^{\operatorname{mix}}=\sum_{r=1}^{R}\Pi_{b_t,r}\,\Delta U^{(r)}_t\in\mathbb{C}^{k}.
\label{eq:DeltaU_mix_t} 
\end{equation}
Here we don't normalize $\Delta U_t^{mix}$ after mixing, since its magnitude reflects the resultant directional strength under the energy-weighted mixture coefficients $\Pi_{b_t,r}$. A small magnitude indicates partial cancellation among competing bases, while a large magnitude indicates reinforcement from coherent bases. We treat this magnitude as an informative control signal for the strength of phase injection. Finally, we concatenate real and imaginary parts as the magnetic phase feature:
\begin{equation}
\label{eq:m_t}
\mathbf{m}_t= \operatorname{Concat}[\operatorname{Re}(\Delta U_t^{\operatorname{mix}}),\operatorname{Im}(\Delta U_t^{\operatorname{mix}})]\in\mathbb{R}^{2k}.
\end{equation}
Thus the time-conditioned directional field is compressed into a low-dimensional per-step signal $\mathbf{m}_t$.

As detailed in Lemma \ref{lem:app_offline_cost}, all spectral decompositions are one-time offline precomputations.

\subsection{Context Feature Embedding}
This module converts discrete context into a $D$-dimensional input sequence. At each step $t$, we use embeddings for POI and category, and incorporate lightweight context: a user embedding $\mathbf{e}_u\in\mathbb{R}^{D}$ and time features (discrete embeddings for hour-of-day and day-of-week, plus a projection of $\log(1+\Delta t_t)$). After concatenation, we apply a linear mapping with a gating function to compress into $D$ dimensions. The resulting backbone sequence $\mathbf{x}_t$ is then fed into the Mag-Mamba layer.

\subsection{Mag-Mamba Layer}
Mamba-3 \cite{anonymous2026mamba} aims to transform trajectory modeling from the similarity aggregation of discrete attention into a controllable continuous-time dynamical recurrence. Its core mechanism involves expressing direction-dependent state evolution via complex rotations, while utilizing time intervals to control memory decay and update magnitude. Building upon the Mamba-3 architecture, we hereby incorporate the magnetic phase feature $\mathbf{m}_t$ as an explicit control signal for these rotational dynamics. Our implementation operates in $\mathbb{R}^D$ as $D/2$ interleaved 2D rotation pairs, which is equivalent to complex multiplication.

Given the backbone input $\mathbf{x}_t\in\mathbb{R}^{D}$, we apply an input projection to obtain
\begin{equation}
    (\mathbf{u}_t,\ \mathbf{B}_t,\ \mathbf{C}_t) = \mathrm{Linear}(\mathbf{x}_t),
\end{equation}
where $\mathbf{u}_t$ is the token content used to parameterize the dynamics and $\mathbf{u}_t,\mathbf{B}_t,\mathbf{C}_t\in\mathbb{R}^{D}$. We map token content and magnetic phase to half dimensions to form angular velocity:
\begin{equation}
\label{eq:theta_t_all}
\begin{gathered}
\boldsymbol{\theta}^{\mathrm{tok}}_t = W_{\theta x}\mathbf{u}_t\in\mathbb{R}^{D/2},\quad\boldsymbol{\theta}^{\mathrm{mag}}_t = W_{\theta m}\mathbf{m}_t\in\mathbb{R}^{D/2}, \\
\boldsymbol{\theta}_t=\boldsymbol{\theta}^{\mathrm{tok}}_t+\boldsymbol{\theta}^{\mathrm{mag}}_t\in\mathbb{R}^{D/2},
\end{gathered}
\end{equation}
where $W_{\theta x}$ and $W_{\theta m}$ are learnable parameters.

Thus spatiotemporal asymmetry enters $\boldsymbol{\theta}_t$ through $\mathbf{m}_t$, determining the direction and magnitude of state rotations. From a dynamical systems perspective (see Appendix \ref{app:control_dynamics}), $\theta_t^{mag}$ acts as an additive topological forcing term on the state's natural oscillation frequency. It explicitly modulates the rotation based on the geometric 'flow' direction, distinct from the content-driven dynamics.

Simultaneously, we map the time gap $\Delta t_t$ to positive step-size vectors via a log transform followed by a Softplus:
\begin{equation}
\boldsymbol{\delta}^{\mathrm{half}}_t=\mathrm{softplus}(W_{\Delta}\,\log(1+\Delta t_t)+\mathbf{b}_{\Delta})\in\mathbb{R}^{D/2}.
\label{eq:delta_t}
\end{equation}
Here, $W_\Delta, \mathbf{b}_{\Delta}$ are learnable weight and bias parameters, respectively.

For consistency with the interleaved 2D rotation pairs, we replicate $\boldsymbol{\delta}^{\mathrm{half}}_t$ across each pair so that the two coordinates within the same rotation pair share an identical step size, yielding a $D$-dimensional step-size vector $\boldsymbol{\delta}_t$ used in the decay and gating terms.

Following Mamba-3~\cite{anonymous2026mamba}, we discretize the continuous-time diagonal SSM using the generalized trapezoidal rule, which yields a two-point update that mixes the previous and current projected inputs via a data-dependent coefficient $\boldsymbol{\lambda}_t$. This A-stable discretization yields the multiplicative decay $\boldsymbol{\alpha}_t$ and the time-scaled input injection terms used below.

Let $\boldsymbol{\rho}\in\mathbb{R}^{D}$ be a learnable diagonal log-parameter and set
\begin{equation}
\label{eq:main_A_P_def}
\mathbf{A}=-\exp(\boldsymbol{\rho})\in\mathbb{R}^{D}.
\end{equation}
We define the discrete decay as
\begin{equation}
\boldsymbol{\alpha}_t=\exp(\boldsymbol{\delta}_t\odot\mathbf{A})\in\mathbb{R}^{D}.
\label{eq:alpha_t}
\end{equation}
We further produce elementwise coefficients via gating:
\begin{equation}
\boldsymbol{\lambda}_t=\sigma(W_{\lambda}\mathbf{u}_t),\quad
\boldsymbol{\beta}_t=(1-\boldsymbol{\lambda}_t)\odot \boldsymbol{\delta}_t\odot \boldsymbol{\alpha}_t,\quad
\boldsymbol{\gamma}_t=\boldsymbol{\lambda}_t\odot \boldsymbol{\delta}_t,
\label{eq:lambda_beta_gamma}
\end{equation}
where $\sigma$ denotes sigmoid.

\paragraph{Rotation operator and state recurrence.}
We assume $D$ is even and interpret $\mathbb{R}^{D}$ as $D/2$ interleaved 2D rotation pairs.
We define the per-pair rotation angle
\begin{equation}
\boldsymbol{\varphi}_t={\boldsymbol{\delta}^{\mathrm{half}}_t}\odot \boldsymbol{\theta}_t\in\mathbb{R}^{D/2}.
\label{eq:phi_t} 
\end{equation}
Let $R_t(\cdot;\boldsymbol{\varphi}_t)$ denote the block-diagonal rotation parameterized by $\boldsymbol{\varphi}_t$. For brevity, we write $R_t(\cdot)\equiv R_t(\cdot;\boldsymbol{\varphi}_t)$ in the recurrence below. For any $\mathbf{v}\in\mathbb{R}^{D}$ and $d=1,\ldots,D/2$, it acts on the $d$-th interleaved pair as:
\begin{equation}
\left[R_t(\mathbf{v};\boldsymbol{\varphi}_t)\right]_{2d-1:2d}=
\begin{bmatrix}
\cos\varphi_{t,d} & -\sin\varphi_{t,d}\\
\sin\varphi_{t,d} & \cos\varphi_{t,d}
\end{bmatrix}
\left[\mathbf{v}\right]_{2d-1:2d}.
\label{eq:R_t_def}
\end{equation}

Let $\mathbf{h}_t\in\mathbb{R}^{D}$ be the hidden state with initialization $\mathbf{h}_0=\mathbf{0}$ and $\mathbf{u}_0=\mathbf{0}$, $\mathbf{B}_0=\mathbf{0}$. For $t=1,\ldots,L$, the recurrence is
\begin{equation}
\mathbf{h}_t=\boldsymbol{\alpha}_t\odot R_t(\mathbf{h}_{t-1})
+\boldsymbol{\beta}_t\odot R_t(\mathbf{B}_{t-1}\odot \mathbf{u}_{t-1})
+\boldsymbol{\gamma}_t\odot (\mathbf{B}_t\odot \mathbf{u}_t).
\label{eq:h_t} 
\end{equation}
The intermediate output is obtained by a readout gate:
\begin{equation}
\mathbf{y}_t=\mathbf{C}_t\odot \mathbf{h}_t.
\label{eq:y_t} 
\end{equation}
Compared with a standard single-input recurrence, the second term explicitly injects a rotated contribution driven by the previous step, enabling both smooth inertial continuation and rapid turning triggered by current input through the gated allocation of $(\boldsymbol{\beta}_t,\boldsymbol{\gamma}_t)$. After SiLU, output projection and residual normalization, we obtain the final sequence representation $\mathbf{Z}_t\in\mathbb{R}^{D}$.

While Eq.~(\ref{eq:h_t}) presents the rotation--decay recurrence for interpretability, we follow the Mamba-3 engineering implementation that realizes the same recurrence via its RoPE-based rotation kernel~\cite{anonymous2026mamba} to preserve parallel efficiency. Please refer to Mamba-3 for further details.

\subsection{Prediction and Loss}
Recommendation is performed by matching the sequence output with candidate POI representations. Let the embedding of POI $p$ be $\mathbf{e}_p\in\mathbb{R}^{D}$. Given $\mathbf{Z}_t\in\mathbb{R}^{D}$, the score for POI $p$ is
\begin{equation}
s_t(p)=\mathbf{Z}_t^\top \mathbf{e}_p.
\end{equation}
We train the model to predict the next POI at each step. Specifically, given $\mathbf{Z}_t$, we predict $p_{t+1}$ and minimize the softmax cross-entropy:
\begin{equation}
\mathcal{L}=-\sum_{t=1}^{L-1}\log\frac{\exp(s_t(p_{t+1}))}{\sum_{p'\in\mathcal{P}}\exp(s_t(p'))}.
\end{equation}

\section{Experiments and Results}
\label{sec:e}

\subsection{Dataset}
We evaluate our proposed model on three datasets collected from two real-world check-in platforms: Foursquare \cite{yang2014modeling} and Gowalla \cite{yuan2013time}. The Foursquare dataset includes two subsets, which are collected from New York City (NYC) in the USA and Tokyo (TKY) in Japan. The Gowalla dataset includes one subset collected from California and Nevada (CA). Their detailed statistics are in Table \ref{tab:data_summary}. The density is calculated as the total number of visits in the training set divided by (number of users × number of POIs), which is used to reflect the sparsity level between users and POIs. We remove POIs with fewer than five check-ins, segment each user's trajectory into sequences of length 3--101, and split the data into training, validation, and test sets in an 8:1:1 ratio.
\begin{table}[ht]
\captionsetup{aboveskip=0pt, belowskip=-3pt}
\centering
\caption{Statistics of the evaluated datasets.}
\renewcommand{\arraystretch}{1.2}
\begin{tabular}{lcccccc}
    \toprule
    & User & POI & Category & Traj & Check-in & Density \\
    \midrule
    \rowcolor{NYCband} NYC & 1,047 & 4,980 & 318 & 13,955 & 101,760 & 0.016 \\
    \rowcolor{TKYband} TKY & 2,281 & 7,832 & 290 & 65,914 & 403,148 & 0.018 \\
    \rowcolor{CAband}  CA  & 3,956 & 9,689 & 296 & 42,982 & 221,717 & 0.005 \\
    \bottomrule
\end{tabular}%
\label{tab:data_summary}
\end{table}

\subsection{Experiment Setting}
We implement our model in PyTorch on an NVIDIA RTX 4090 GPU. We optimize it with Adam using a learning rate of $1\times 10^{-3}$ and weight decay of $1\times 10^{-3}$; the batch size is 128 and we train for 50 epochs. The hidden size $D$ is 96, the time embedding size is 32, and we use 2 Mag-Mamba layers.

For the magnetic phase encoder, we set the graph radius to 1.5 km, $\alpha=1$, $\kappa=1$, and $\sigma_{\mathrm{geo}}=1$ km; the phase dimension is $k=16$. We use 168 hour-of-week bins and set the low-rank temporal basis size to $R=12$. The magnetic charge is $q=0.20$ for NYC and TKY, and $q=0.15$ for CA due to the varying graph densities across datasets. We report NDCG@$k$ and MRR with $k\in\{1,5,10\}$, and compute all transition statistics and factorizations on the training split only.

\begin{table*}[t]
\captionsetup{aboveskip=0pt, belowskip=-1pt}
\centering
\caption{Overall performance comparison with baseline models on three datasets.
\textbf{Top\mbox{-}1},\; \textcolor{Top2Gray}{\underline{Top\mbox{-}2}}.}
\setlength{\tabcolsep}{2.5pt} 
\renewcommand{\arraystretch}{1.15}
\resizebox{\textwidth}{!}{%
    \begin{tabular}{l *{4}{Y} *{4}{T} *{4}{C}}
    \toprule

    \multirow{2.5}{*}{\textbf{Method}} &
    \multicolumn{4}{c}{\cellcolor{NYCband}\textbf{NYC}} &
    \multicolumn{4}{c}{\cellcolor{TKYband}\textbf{TKY}} & 
    \multicolumn{4}{c}{\cellcolor{CAband}\textbf{CA}} \\ 
    \cmidrule(lr){2-5}\cmidrule(lr){6-9}\cmidrule(lr){10-13}
    & \hNYC{ND@1} & \hNYC{ND@5} & \hNYC{ND@10} & \hNYC{MRR}
    & \hTKY{ND@1} & \hTKY{ND@5} & \hTKY{ND@10} & \hTKY{MRR}
    & \hCA{ND@1}  & \hCA{ND@5}  & \hCA{ND@10}  & \hCA{MRR} \\
    \midrule

    LSTM         & 0.1306 & 0.2336 & 0.2585 & 0.2259
                 & 0.1110 & 0.2233 & 0.2496 & 0.1952
                 & 0.0864 & 0.1459 & 0.1711 & 0.1554 \\
    PLSPL        & 0.1601 & 0.3048 & 0.3336 & 0.2849
                 & 0.1495 & 0.2831 & 0.3143 & 0.2642
                 & 0.1084 & 0.1759 & 0.2029 & 0.1678 \\
    HME          & 0.1619 & 0.2806 & 0.3226 & 0.2787
                 & 0.1535 & 0.2637 & 0.2924 & 0.2366
                 & 0.1181 & 0.1886 & 0.2232 & 0.1945 \\
    GETNext      & 0.2244 & 0.3736 & 0.4046 & 0.3472
                 & 0.1767 & 0.3072 & 0.3297 & 0.2934
                 & 0.1342 & 0.2188 & 0.2468 & 0.2121 \\
    AGRAN        & 0.2202 & 0.3638 & 0.3792 & 0.3343
                 & 0.1755 & 0.2989 & 0.3261 & 0.2879
                 & 0.1329 & 0.2121 & 0.2331 & 0.2043 \\
    ROTAN        & 0.2192 & 0.3652 & 0.3944 & 0.3468
                 & 0.1822 & 0.3165 & 0.3467 & 0.3186
                 & 0.1345 & 0.2173 & 0.2564 & 0.2256 \\
    MCLP         & \second{0.2404} & 0.3674 & 0.3973 & 0.3507
                 & 0.1662 & 0.3110 & 0.3415 & 0.3199
                 & 0.1324 & 0.1914 & 0.2121 & 0.1895 \\
    MTNet        & 0.2346 & 0.3689 & 0.4013 & \second{0.3522}
                 & \second{0.2107} & \second{0.3345} & 0.3612 & 0.3191
                 & \second{0.1401} & \second{0.2346} & 0.2617 & \second{0.2387} \\
    GeoMamba     & 0.1988 & 0.3392 & 0.3506 & 0.3246
                 & 0.1851 & 0.2953 & 0.3205 & 0.2858
                 & 0.1256 & 0.2029 & 0.2215 & 0.1962 \\
    HMST         & 0.2138 & \second{0.3747} & \second{0.4063} & 0.3482
                 & 0.1925 & 0.3325 & \second{0.3690} & \second{0.3257}
                 & 0.1356 & 0.2325 & \second{0.2680} & 0.2300 \\
    HVGAE        & 0.2271 & 0.3651 & 0.3982 & 0.3470
                 & 0.1977 & 0.3167 & 0.3455 & 0.3180
                 & 0.1391 & 0.2325 & 0.2658 & 0.2367 \\

    \midrule
    \cellcolor{black!4}\textbf{Mag-Mamba} &
    \best{0.2451} & \best{0.3823} & \best{0.4149} & \best{0.3603}
    & \best{0.2417} & \best{0.3602} & \best{0.3889} & \best{0.3455}
    & \best{0.1595} & \best{0.2528} & \best{0.2761} & \best{0.2462} \\
    \rowcolor{black!2}
    \textit{improv.} &
    \multicolumn{1}{c}{\up{+1.96\%}} & \multicolumn{1}{c}{\up{+2.03\%}} & \multicolumn{1}{c}{\up{+2.12\%}} & \multicolumn{1}{c}{\up{+2.30\%}}
    & \multicolumn{1}{c}{\up{+14.71\%}} & \multicolumn{1}{c}{\up{+7.68\%}} & \multicolumn{1}{c}{\up{+5.39\%}} & \multicolumn{1}{c}{\up{+6.08\%}}
    & \multicolumn{1}{c}{\up{+13.85\%}} & \multicolumn{1}{c}{\up{+7.76\%}} & \multicolumn{1}{c}{\up{+3.02\%}} & \multicolumn{1}{c}{\up{+3.14\%}} \\
    \bottomrule
    \end{tabular}
}
\label{tab:performance_metrics}
\end{table*}
\subsection{Baseline Model}
We compare Mag-Mamba against 11 baselines:
\textbf{LSTM}~\cite{hochreiter1997long} and \textbf{PLSPL}~\cite{wu2020personalized}, which utilize RNNs and attention for sequential preference learning;
\textbf{HME}~\cite{feng2020hme}, embedding check-ins into a Poincaré ball;
\textbf{GETNext}~\cite{yang2022getnext}, which incorporates global trajectory flow maps;
\textbf{AGRAN}~\cite{wang2023adaptive}, employing adaptive graph structure learning;
\textbf{ROTAN}~\cite{feng2024rotan}, employing a Time2Rotation technique to encode temporal information as rotations;
\textbf{MCLP}~\cite{sun2024going}, utilizing topic modeling and multi-head attention;
\textbf{MTNet}~\cite{huang2024learning}, constructing a hierarchical Mobility Tree to capture multi-granularity time slot preferences;
\textbf{GeoMamba}~\cite{chen2024geomamba}, leveraging Mamba's linear complexity with a hierarchical geography encoder;
\textbf{HMST}~\cite{qiao2025hyperbolic}, employing hyperbolic rotations to jointly model hierarchical structures and multi-semantic transitions;
\textbf{HVGAE}~\cite{liu2025hyperbolic}, integrating a Hyperbolic GCN and Variational Graph Auto-Encoder with Rotary Position Mamba to capture hierarchical POI relationships.
\subsection{Performance Comparison With Baselines}
We first compare our model with 11 baseline models (Table \ref{tab:performance_metrics}). Our model consistently outperforms all 11 baselines across three datasets, demonstrating robustness in capturing spatiotemporal patterns.

Among the baselines, Transformer-based methods (GETNext, MCLP) outperform traditional models (LSTM, PLSPL, HME). This is attributed to their multi-head attention mechanisms, which effectively capture sequence patterns and multi-contextual information, significantly enhancing the modeling of complex sequential dependencies. AGRAN, via its adaptive graph structure, captures geographic dependencies more effectively than traditional GCNs, yielding competitive results. Regarding Mamba-based approaches, GeoMamba maintains a balance between efficiency and accuracy, benefiting from its geographic encoding module and the efficient state update of the Mamba model. HVGAE and HMST achieve improvements by modeling deep POI relationships via hyperbolic graph convolutions and variational inference, respectively. However, while both methods utilize "rotation" mechanisms, their rotations are primarily for embedding alignment rather than explicitly characterizing spatiotemporal asymmetry. Similarly, while ROTAN effectively models rotation in the temporal dimension, it neglects the topological constraints of geographic space.

Overall, Mag-Mamba leverages complex-valued phase rotation to model spatiotemporal asymmetry in a unified manner, consistently outperforming all baselines across three datasets. This comprehensive superiority highlights the effectiveness of coupling magnetic geometric priors with complex-valued dynamics for next POI recommendation.

\begin{table}[ht]
\captionsetup{aboveskip=0pt, belowskip=-3pt}
\centering
\caption{Performance comparison of model variants.}
\setlength{\tabcolsep}{4pt}
\renewcommand{\arraystretch}{1.18}
\small
\resizebox{0.48\textwidth}{!}{
\begin{tabular}{l *{3}{T} *{3}{C}}
\toprule
\multirow{2.5}{*}{\textbf{Model}} &
\multicolumn{3}{c}{\cellcolor{TKYband}\textbf{TKY}} &
\multicolumn{3}{c}{\cellcolor{CAband}\textbf{CA}} \\
\cmidrule(lr){2-4}\cmidrule(lr){5-7}
& \hTKY{ND@5} & \hTKY{ND@10} & \hTKY{MRR}
& \hCA{ND@5}  & \hCA{ND@10}  & \hCA{MRR} \\
\midrule
w/o Mamba \rmv & 0.3096 & 0.3335 & 0.2862 & 0.2046 & 0.2267 & 0.1998 \\
w/o Spatial  \rmv & 0.3314 & 0.3560 & 0.3117 & 0.2122 & 0.2348 & 0.2085 \\
w/o Phase  \rmv & 0.3467 & 0.3728 & 0.3248 & 0.2340 & 0.2588 & 0.2295 \\
Real-Mamba \rmv & 0.3550 & 0.3834 & 0.3405 & 0.2398 & 0.2631 & 0.2359 \\
MLP-Eigen \rmv & 0.3491 & 0.3743 & 0.3279 & 0.2233 & 0.2575 & 0.2290 \\
w/o TC \rmv & 0.3525 & 0.3763 & 0.3372 & 0.2443 & 0.2672 & 0.2396 \\
\midrule
\cellcolor{FullRow}\textbf{Full model} &
\bestab{0.3602} & \bestab{0.3889} & \bestab{0.3455} &
\bestab{0.2528} & \bestab{0.2761} & \bestab{0.2462} \\
\bottomrule
\end{tabular}
}
\label{ablation_results}
\end{table}

\subsection{Ablation Study}
We conduct a comprehensive ablation study with six variants:
(1) \textbf{w/o Mamba} replaces the Mamba with an MLP to test the necessity of sequential modeling;
(2) \textbf{w/o Spatial} removes the geographic Magnetic Laplacian encoder to verify the impact of geographic priors;
(3) \textbf{w/o Phase} sets all phases to zero, i.e., $\Phi^{(r)}_{ij}=0$ so that $H^{(r)}_{ij}=W_{ij}$, and the resulting $m_t$ contains no directional signal;
(4) \textbf{Real-Mamba} replaces the Complex Mamba with a standard real-valued Mamba to validate the necessity of complex-domain rotational dynamics;
(5) \textbf{w/o TC (Time-Cond)} removes the time-dependent modulation mechanism from the Magnetic Laplacian to examine the model's ability to capture dynamic tidal patterns and (6) \textbf{MLP-Eigen} substitutes the Magnetic Laplacian's eigenvectors with MLP-generated features to validate the necessity of explicit global spectral information.

The results are reported in Table \ref{ablation_results} (ND@5, ND@10, and MRR). Using the full model as the baseline, we observe the following: First, the w/o Mamba variant suffers the most significant performance degradation, confirming that sequential transition patterns in user trajectories are fundamental to personalized recommendation. Second, the substantial drop in w/o Spatial validates that geographic encoding serves as indispensable context for POI recommendation. A critical finding is the performance decay in w/o Phase. Removing phases collapses the model into an isotropic undirected graph, losing the spectral signature required to distinguish reciprocal flows (e.g., commuting vs. returning). This proves that simple undirected proximity is insufficient to describe complex urban mobility. Furthermore, Real-Mamba underperforms because it lacks the rotational operators native to the complex domain, preventing the effective evolution of phase features. This emphasizes that "complex features" must be coupled with "complex dynamics" to maximize utility. Additionally, the decline in MLP-Eigen indicates that global topological consistency requires explicit spectral derivation and is difficult to be approximated by local MLPs. Finally, w/o Time-Cond demonstrates that lacking time perception prevents the model from capturing topological reversals during rush hours, thereby limiting the upper bound of tidal pattern prediction.
\subsection{Sensitivity Analysis}

\begin{figure}
\captionsetup{skip=3pt,belowskip=-15pt}
\centering
\includegraphics[width=\linewidth]{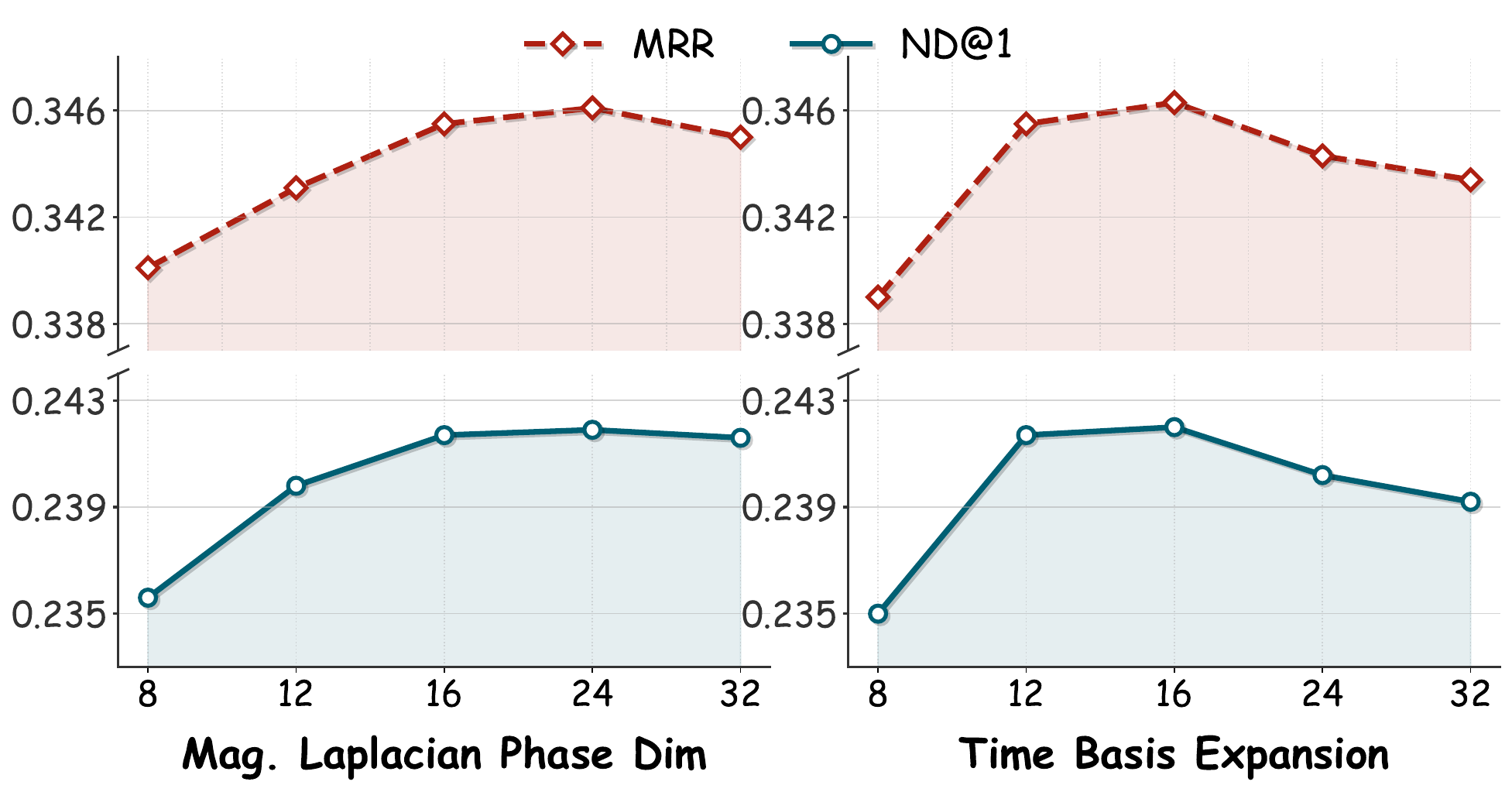}
\caption{Sensitivity analysis of the magnetic phase dimension (Left) and the low-rank temporal bases (Right) on the TKY dataset.}
\label{sen1}
\end{figure}

\begin{figure}
\captionsetup{skip=3pt,belowskip=-15pt}
\centering
\includegraphics[width=\linewidth]{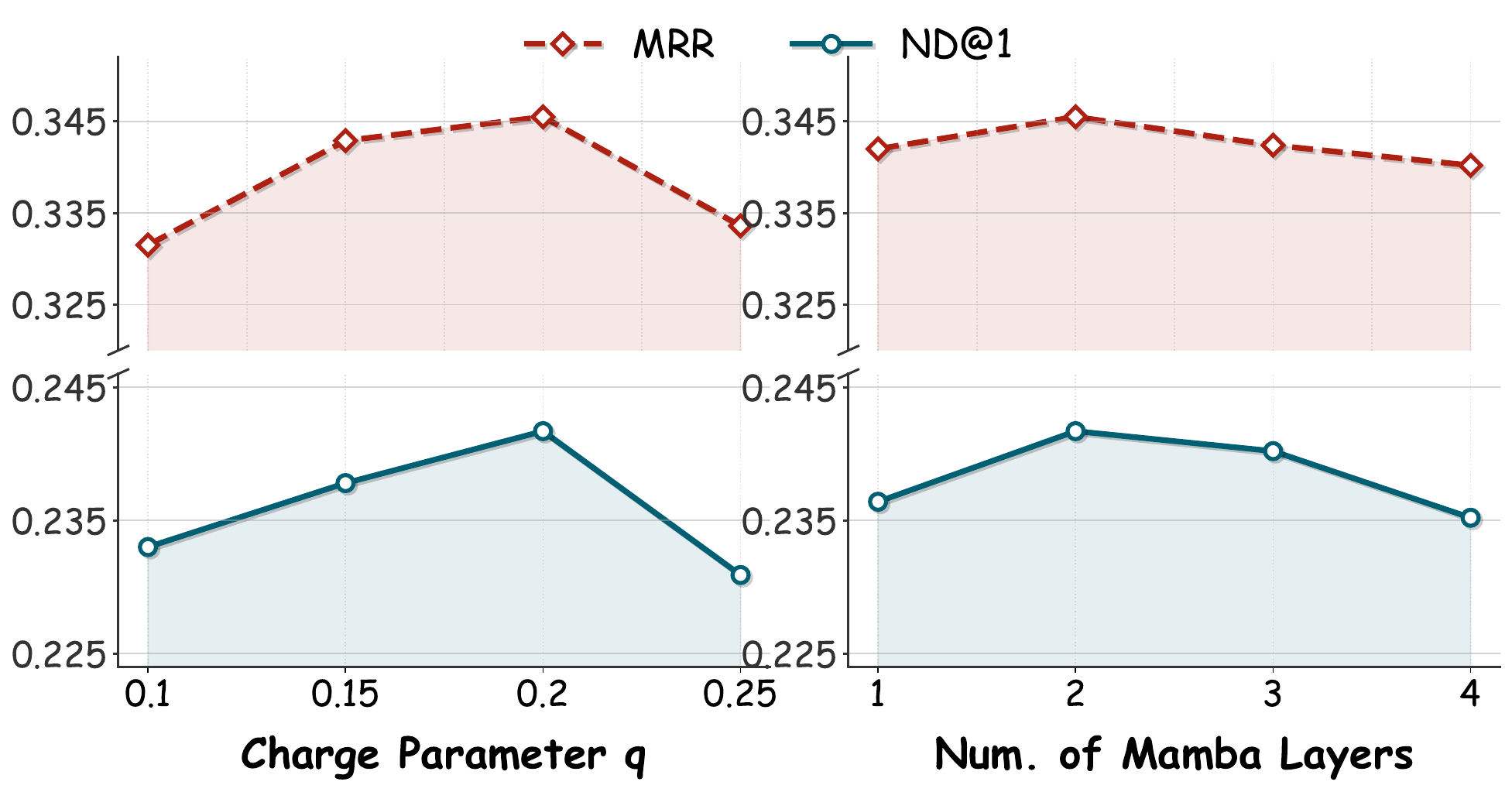}
\caption{Sensitivity analysis of the magnetic charge parameter $q$ (Left) and the Mamba layers (Right) on the TKY dataset.}
\label{sen2}
\end{figure}

We evaluate the impact of key hyperparameters on the TKY dataset using ND@1 and MRR. Results for the Magnetic Laplacian phase dimension and low-rank temporal bases are shown in Figure \ref{sen1}; results for the charge parameter $q$ and Mamba layers are shown in Figure \ref{sen2}.
\paragraph{Magnetic Phase Embedding Dimension}
This dimension determines the capacity of the spectral space. We tested candidate values $\{8, 12, 16, 24, 32\}$. As shown in Figure \ref{sen1} left, performance improves with dimension size but yields only marginal gains beyond 16, while training time and memory usage increase significantly. Therefore, we set the dimension to 16.
\paragraph{Number of Low-Rank Temporal Bases ($R$)}
The number of bases determines the model's coverage of underlying mobility patterns. Increasing from 4 to 12 yields distinct gains, while the performance saturates from 12 to 16. This suggests that urban mobility patterns can be decomposed into a finite set of principal components. Too few bases lead to underfitting, while too many yield diminishing returns and potential overfitting.
\paragraph{Magnetic Charge Parameter ($q$)}
The parameter $q$ controls the strength of directionality modeling. The results show a unimodal distribution. When $q$ is too small, the model degenerates to an approximation of an undirected graph with insufficient directional information. Conversely, when $q$ is too large, it may introduce unstable directed noise (phase aliasing). The optimal value is observed at 0.2 for the TKY dataset.
\paragraph{Number of Mamba Layers}
We explored model performance with layers $\{1, 2, 3, 4\}$. As the layer count increases, the model's expressive power improves, peaking at 2 layers. However, beyond 2 layers, performance begins to decline due to over-smoothing. Thus, we select 2 as the optimal number of Mamba layers.

\subsection{Performance on Asymmetric Subgroups}
\begin{figure}
\captionsetup{skip=3pt,belowskip=-10pt}
\centering
\includegraphics[width=\linewidth]{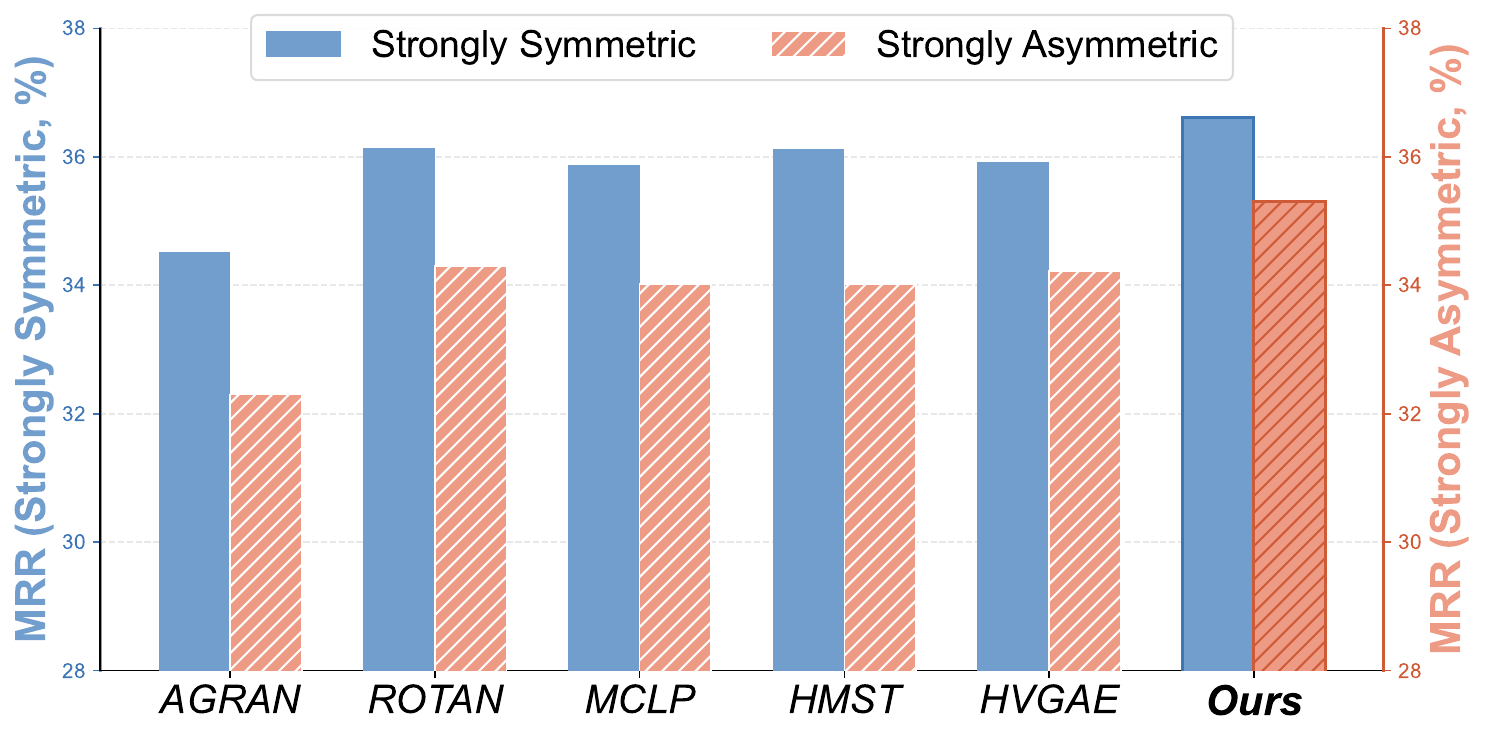}
\caption{Performance comparison (MRR) between Strongly Symmetric and Strongly Asymmetric subgroups on the TKY dataset.}
\label{asym_group}
\end{figure}

To verify the adaptability of Mag-Mamba to asymmetric scenarios, we designed an asymmetry-stratified evaluation on the TKY dataset. We define the asymmetry coefficient for a POI pair $(i, j)$ as:
$\mathrm{Asy}(i,j)=\frac{1}{N_b}\sum_{b=0}^{N_b-1}\left|P(j\mid i,b)-P(i\mid j,b)\right|.$
To reduce noise from sparse transitions, we compute $\mathrm{Asy}(i,j)$ only for POI pairs whose total bidirectional transition count exceeds a minimum threshold ($T_{min}=20$). The test samples are sorted by this coefficient and divided into tertiles: the lowest third forms the Strongly Symmetric Group, and the highest third forms the Strongly Asymmetric Group. We then compare the performance (MRR) of Mag-Mamba against baselines on these two groups.

The results are illustrated in Figure \ref{asym_group}. We observe that all methods, including baselines, exhibit lower absolute metrics in the Strongly Asymmetric Group compared to the Strongly Symmetric Group, confirming that asymmetric flows (e.g., tidal commuting) present a higher reasoning difficulty. However, Mag-Mamba exhibits a significantly smaller performance drop in the asymmetric group compared to baselines. This demonstrates its superior robustness in capturing irregular mobility flows and time-conditioned directional preferences. This result strongly validates Mag-Mamba's advantage in handling high-asymmetry scenarios, successfully transforming the challenge of asymmetry into a model feature.
\subsection{Case Study}

\begin{figure}
\captionsetup{skip=3pt,belowskip=-10pt}
\centering
\includegraphics[width=\linewidth]{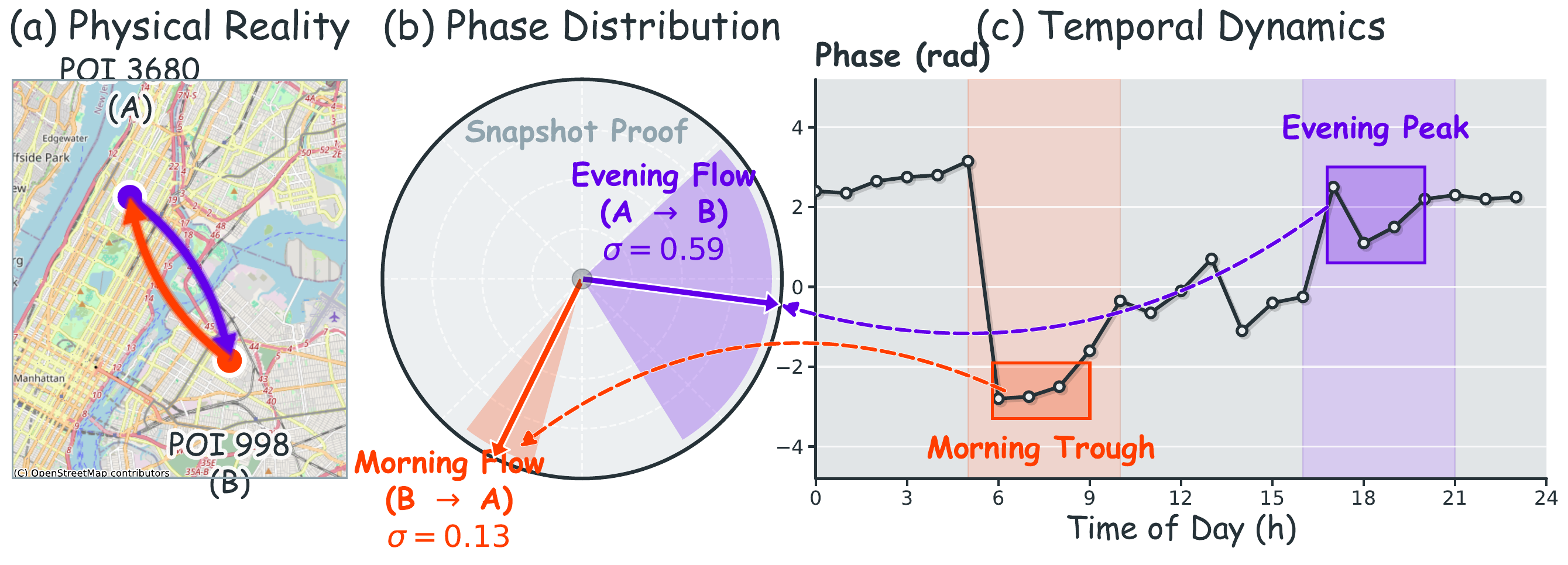}
\caption{Case study on tidal commuting in NYC. (Left) Spatial trajectories; (Middle) Morning vs. Evening phase vectors; (Right) 24-hour phase angle evolution.}
\label{case}
\end{figure}

We selected a sample with a high asymmetry index from the NYC dataset for visual analysis. This sample corresponds to a typical functional zone transition (commute): daytime movement flows from a residential area to a business district, while nighttime movement returns from the office to the residence. As shown in Figure \ref{case} (Left), the two paths are geographically proximal but directionally opposite.

Further observation reveals that these distinct morning and evening flow preferences correspond to different phase differences, creating time-dependent directional driving forces on the same geographic edge (Figure \ref{case} Middle). Additionally, by tracking the phase angle variation of this high-frequency transition edge over time, we observe a systematic shift in the phase angle (Figure \ref{case} Right). This shift enables the model to assign distinct directional updates to the same geographic connection during morning and evening rush hours, thereby accurately distinguishing between commuting and returning behaviors.

\section{Conclusion}
In this paper, we propose Mag-Mamba, a novel framework that effectively addresses the challenge of spatiotemporal asymmetry in next POI recommendation. Unlike existing methods that rely on symmetric operators in the real domain, we present the modeling of directional mobility flows as phase-driven rotational dynamics in the complex domain. Specifically, by introducing a Time-conditioned Magnetic Phase Encoder, we encode the anisotropic graph topology into complex phase features. Subsequently, we design a Complex-valued Mamba to dynamically evolve these features, enabling the model to capture complex spatiotemporal patterns such as flow reversals and tidal commuting. Extensive experiments on three real-world datasets demonstrate that Mag-Mamba consistently outperforms state-of-the-art baselines. Future work will explore extending this complex-domain modeling paradigm to other spatiotemporal tasks, such as traffic flow forecasting.

\bibliographystyle{ACM-Reference-Format}
\bibliography{sample-base}

\appendix
\setcounter{equation}{0}
\setcounter{figure}{0}
\setcounter{table}{0}
\renewcommand{\theequation}{A\arabic{equation}}
\renewcommand{\thefigure}{A\arabic{figure}}
\renewcommand{\thetable}{A\arabic{table}}
\section*{Notation (Quick Reference)}
\label{app:notation}
\addcontentsline{toc}{section}{Notation (Quick Reference)}
\vspace{-4pt}

\begin{small}
\setlength{\tabcolsep}{4pt}
\renewcommand{\arraystretch}{1.08}
\begin{tabular}{p{0.15\linewidth}p{0.78\linewidth}}
\toprule
\textbf{Symbol} & \textbf{Meaning} \\
\midrule

$\mathcal{U}$ & User set. \\
$\mathcal{P}$ & POI set. \\
$L$ & Sequence length. \\
$p_t$ & POI at step $t$. \\
$\tau_t$ & Timestamp at step $t$. \\
$\Delta t_t$ & Time gap: $\Delta t_1=0$, $\Delta t_t=\tau_t-\tau_{t-1}$ for $t\ge2$. \\
$N_b$ & Number of discrete time bins. \\
$b_t$ & Time-bin index at step $t$. \\

$\mathcal{G}$ & Geographic graph. \\
$W_{ij}$ & Symmetric geographic weight on $\{i,j\}$. \\

$R$ & Rank, number of spatiotemporal direction bases. \\
$\Pi$ & Time-bin mixing coefficients, $\Pi\in\mathbb{R}^{N_b\times R}$. \\
$\Psi$ & Edge basis coefficients, $\Psi\in\mathbb{R}^{R\times|\mathcal{E}|}$. \\
$q$ & Magnetic charge. \\
$H^{(r)}$ & Hermitian complex adjacency, $H^{(r)}_{ij}=W_{ij}\exp(i\Phi^{(r)}_{ij})$. \\
$L^{(r)}$ & Normalized magnetic Laplacian. \\

$k$ & Number of kept eigenvectors per basis. \\
$U^{(r)}$ & Unit-modulus phase tokens from eigenvector phases, $U^{(r)}\in\mathbb{C}^{|\mathcal{P}|\times k}$. \\
$m_t$ & Magnetic phase feature input to Mamba, $m_t\in\mathbb{R}^{2k}$. \\

$D$ & Hidden size of Mag-Mamba. \\
$x_t$ & Backbone input at step $t$, $x_t\in\mathbb{R}^{D}$. \\

${\boldsymbol{\delta}^{\mathrm{half}}_t}$ & Positive step-size for decay, ${\boldsymbol{\delta}^{\mathrm{half}}_t}=\mathrm{softplus}(W_\Delta\log(1+\Delta t_t)+b_\Delta)$. \\

$\boldsymbol{\theta}_t$ & Total angular velocity, $\boldsymbol{\theta}_t=\boldsymbol{\theta}^{\mathrm{tok}}_t+\boldsymbol{\theta}^{\mathrm{mag}}_t\in\mathbb{R}^{D/2}$. \\
$\boldsymbol{\varphi}_t$ & Rotation angles, $\boldsymbol{\varphi}_t={\boldsymbol{\delta}^{\mathrm{half}}_t\odot \boldsymbol{\theta}_t}\in\mathbb{R}^{D/2}$. \\
$R_t(\cdot)$ & Block-diagonal 2D rotation operator parameterized by $\boldsymbol{\varphi}_t$. \\

$h_t$ & Hidden state, $h_t\in\mathbb{R}^{D}$. \\
$Z_t$ & Final representation used for prediction, $Z_t\in\mathbb{R}^{D}$. \\

\bottomrule
\end{tabular}
\end{small}
\vspace{6pt}

\makeatletter
\@ifundefined{lemma}{\newtheorem{lemma}{Lemma}}{}
\@ifundefined{proposition}{\newtheorem{proposition}{Proposition}}{}
\@ifundefined{corollary}{\newtheorem{corollary}{Corollary}}{}
\@ifundefined{remark}{\newtheorem{remark}{Remark}}{}
\makeatother

\section{Efficiency Analysis}

Figure~\ref{eff} reports end-to-end inference latency and throughput as trajectory length $L$ increases.
GeoMamba is the fastest, while Mag-Mamba is consistently the second-fastest and maintains low tail latency. In contrast, AGRAN shows a clear latency blow-up, especially in p95/p99, and other graph-based baselines are generally slower. Throughput decreases in trajectories/s as $L$ grows, while steps/s is stable or slightly increases.

Overall, the Magnetic Phase Encoder adds only modest overhead and preserves efficient scaling.

\begin{figure*}
\captionsetup{skip=2pt,belowskip=-5pt}
 \centering
 \includegraphics[width=\textwidth]{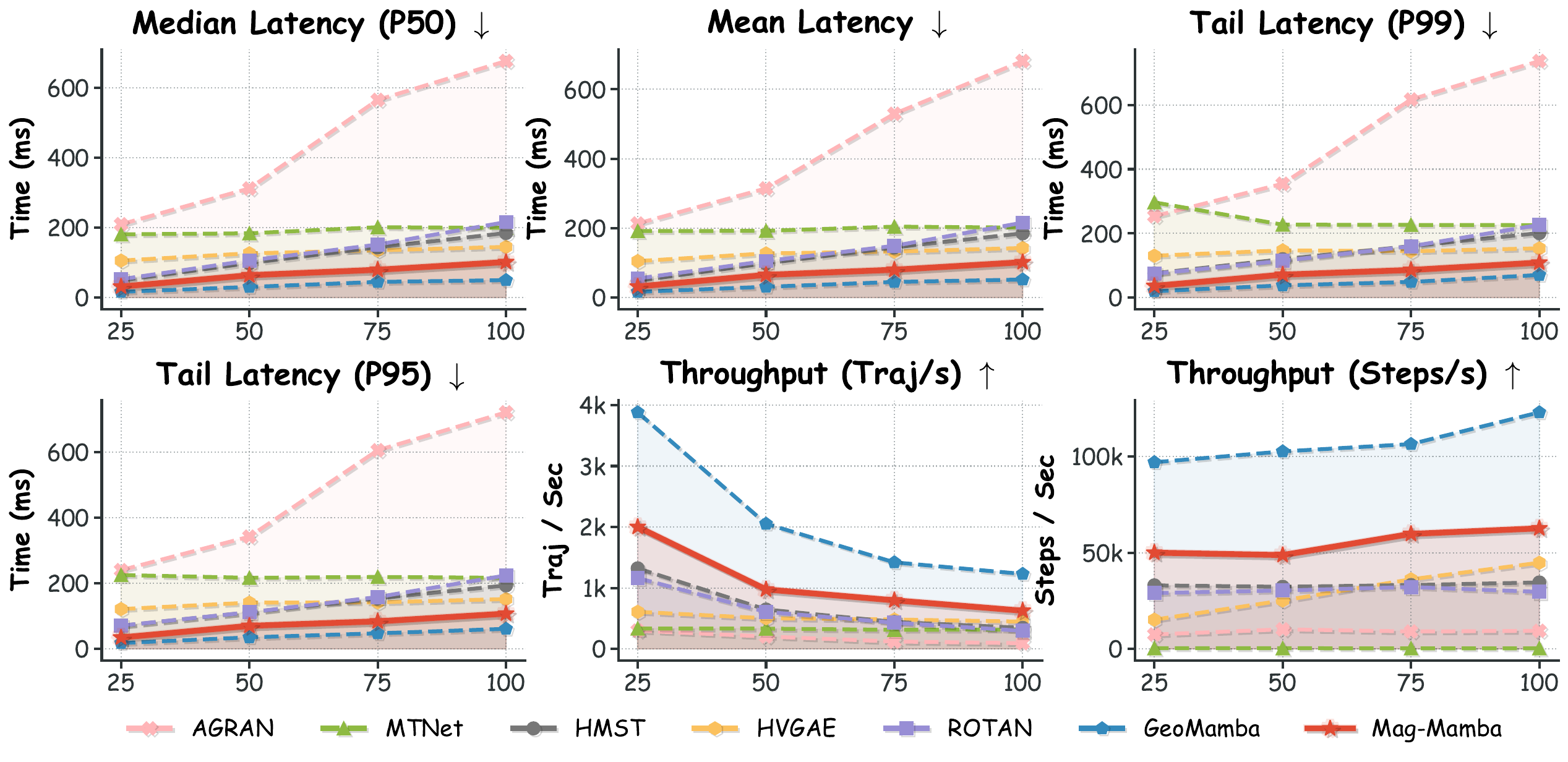}
\caption{\textbf{Efficiency benchmark.} End-to-end inference latency (mean/p50/p95/p99, lower is better) and throughput (trajectories/s and steps/s, higher is better) versus trajectory length $L\in\{25,50,75,100\}$ on a single GPU with FP16 AMP. We use batch size $B{=}128$ with 20 warmup iterations and report statistics over 200 timed iterations. Latency is measured per batch in milliseconds.}
\label{eff}
 \Description[none]{}
\end{figure*}

\section{Additional Proofs and Analyses}
\label{app:proofs}

\subsection{Preliminaries: Dynamics Discretization (Recap)}
\label{app:recap_discretization}
Following Mamba-3~\cite{anonymous2026mamba}, we adopt the generalized trapezoidal discretization for a diagonal continuous-time SSM.
For completeness, we restate the resulting coefficients used in the Mag-Mamba update.

Let $\boldsymbol{\rho}\in\mathbb{R}^{D}$ be a learnable diagonal log-parameter and set the transformation matrix $\mathbf{A}$ as defined in the main text (see Equation~\ref{eq:main_A_P_def}):
\begin{equation}
\mathbf{A} = -\exp(\boldsymbol{\rho}) \in \mathbb{R}^{D}.
\end{equation}
Given the time gap $\Delta t_t$, we map it to a positive step-size vector (half-dimensional, $D/2$) via a log transform followed by a Softplus:
\begin{equation}
\boldsymbol{\delta}^{\mathrm{half}}_t = \mathrm{softplus}(W_{\Delta}\,\log(1+\Delta t_t)+\mathbf{b}_{\Delta})\in\mathbb{R}^{D/2}
\end{equation}
Here, $W_{\Delta}, \mathbf{b}_{\Delta}$ are learnable weight and bias parameters, respectively.

We replicate the $D/2$-dimensional $\boldsymbol{\delta}^{\mathrm{half}}_t$ across each rotation pair to yield the $D$-dimensional $\boldsymbol{\delta}_t$, which is used for both decay and rotation.

The multiplicative decay is
\begin{equation}
\label{eq:app_alpha_def}
\boldsymbol{\alpha}_t=\exp(\boldsymbol{\delta}_t\odot\mathbf{A})\in\mathbb{R}^{D}.
\end{equation}
Let $\mathbf{u}_t$ denote the token content driving the gates; then
\begin{equation}
\label{eq:app_lambda_beta_gamma}
\begin{gathered}
\boldsymbol{\lambda}_t = \sigma(W_{\lambda}\mathbf{u}_t),\\
\boldsymbol{\beta}_t = (1-\boldsymbol{\lambda}_t)\odot \boldsymbol{\delta}_t\odot \boldsymbol{\alpha}_t,\quad
\boldsymbol{\gamma}_t=\boldsymbol{\lambda}_t\odot \boldsymbol{\delta}_t.
\end{gathered}
\end{equation}
The per-pair rotation angle is
\begin{equation}
\label{eq:app_phi_def}
\boldsymbol{\varphi}_t=\boldsymbol{\delta}^{\mathrm{half}}_t\odot \boldsymbol{\theta}_t\in\mathbb{R}^{D/2},
\end{equation}
\begin{equation}
\label{eq:app_theta_t}
\boldsymbol{\theta}_t=\boldsymbol{\theta}^{\mathrm{tok}}_t+\boldsymbol{\theta}^{\mathrm{mag}}_t.
\end{equation}
Define $R_t(\cdot)$ as the block-diagonal 2D rotation on $D/2$ interleaved pairs, rotating the $d$-th pair by angle $\varphi_{t,d}$.

Then the Mag-Mamba recurrence is
\begin{align*}
\mathbf{h}_t &= \boldsymbol{\alpha}_t\odot R_t(\mathbf{h}_{t-1})
+\boldsymbol{\beta}_t\odot R_t(\mathbf{B}_{t-1}\odot \mathbf{u}_{t-1}) \\
&\quad +\boldsymbol{\gamma}_t\odot (\mathbf{B}_t\odot \mathbf{u}_t) \\
\mathbf{y}_t &= \mathbf{C}_t\odot \mathbf{h}_t
\end{align*}

\subsection{Magnetic Laplacian: Hermitian Property and Spectral Well-posedness}
\label{app:hermitian}

The geographic weight matrix $W\in\mathbb{R}^{|\mathcal{P}|\times|\mathcal{P}|}$ is symmetric with $W_{ij}=W_{ji}\ge 0$ and $W_{ii}=0$.
For each basis $r$, the direction field $A^{(r)}\in\mathbb{R}^{|\mathcal{P}|\times |\mathcal{P}|}$ is constructed from edge coefficients $\Psi_{r,e}$ by
\begin{equation}
\label{eq:app_A_construct}
A^{(r)}_{ij}=
\begin{cases}
\ \ \Psi_{r, e_{ij}}, & \{i,j\}\in\mathcal{E},\ i<j,\\
-\,\Psi_{r, e_{ij}}, & \{i,j\}\in\mathcal{E},\ i>j,\\
\ \ 0, & \text{otherwise},
\end{cases}
\end{equation}
and the magnetic phases and complex adjacency are
\begin{equation}
\label{eq:app_phi_H_def}
\Phi^{(r)}_{ij}=2\pi q\,A^{(r)}_{ij},
\qquad
H^{(r)}_{ij}=W_{ij}\exp(\mathrm{i}\Phi^{(r)}_{ij}),
\quad \mathrm{i}=\sqrt{-1}.
\end{equation}
Let $D_g=\mathrm{diag}(d_1,\ldots,d_{|\mathcal{P}|})$ with $d_i=\sum_j W_{ij}$ and define the normalized magnetic Laplacian
\begin{equation}
\label{eq:app_L_def}
L^{(r)} = I_{|\mathcal{P}|} - D_g^{-1/2}H^{(r)}D_g^{-1/2}.
\end{equation}

\begin{lemma}[Antisymmetry $\Rightarrow$ Hermitian adjacency]
\label{lem:app_H_hermitian}
For each basis $r$, the matrix $A^{(r)}$ defined in~\eqref{eq:app_A_construct} is antisymmetric: $(A^{(r)})^\top=-A^{(r)}$.
Consequently, $\Phi^{(r)}_{ji}=-\Phi^{(r)}_{ij}$ and the complex adjacency $H^{(r)}$ defined in~\eqref{eq:app_phi_H_def} is Hermitian: $(H^{(r)})^\ast = H^{(r)}$, i.e., $H^{(r)}_{ji}=\overline{H^{(r)}_{ij}}$.
\end{lemma}
\begin{proof}
By construction, for any $(i,j)$ with $\{i,j\}\in\mathcal{E}$, swapping indices flips the sign in~\eqref{eq:app_A_construct}, hence $A^{(r)}_{ji}=-A^{(r)}_{ij}$; for non-edges both entries are zero. Thus $(A^{(r)})^\top=-A^{(r)}$ and $\Phi^{(r)}_{ji}=-\Phi^{(r)}_{ij}$.

Since $W_{ij}=W_{ji}\in\mathbb{R}$ and $\exp(\mathrm{i}\Phi^{(r)}_{ji})=\exp(-\mathrm{i}\Phi^{(r)}_{ij})=\overline{\exp(\mathrm{i}\Phi^{(r)}_{ij})}$, we have
\begin{equation}
\begin{split}
H^{(r)}_{ji} &= W_{ji}\exp(\mathrm{i}\Phi^{(r)}_{ji})
=W_{ij}\,\overline{\exp(\mathrm{i}\Phi^{(r)}_{ij})} \\
&=\overline{W_{ij}\exp(\mathrm{i}\Phi^{(r)}_{ij})}
=\overline{H^{(r)}_{ij}}.
\end{split}
\label{eq:app_H_conjugate}
\end{equation}
Therefore $H^{(r)}$ is Hermitian.
\end{proof}

\begin{lemma}[Hermitian PSD of the magnetic Laplacian]
\label{lem:app_L_psd}
For each basis $r$, the normalized magnetic Laplacian $L^{(r)}$ in~\eqref{eq:app_L_def} is Hermitian positive semidefinite.
In particular, for any $\mathbf{x}\in\mathbb{C}^{|\mathcal{P}|}$,
\begin{equation}
\label{eq:app_quadratic_form}
\mathbf{x}^\ast L^{(r)} \mathbf{x}
=\frac12\sum_{i,j} W_{ij}\left|
\frac{x_i}{\sqrt{d_i}}-e^{\mathrm{i}\Phi^{(r)}_{ij}}\frac{x_j}{\sqrt{d_j}}
\right|^2
\ \ge\ 0.
\end{equation}
\end{lemma}
\begin{proof}
By Lemma~\ref{lem:app_H_hermitian}, $H^{(r)}$ is Hermitian; since $D_g^{-1/2}$ is real diagonal, $D_g^{-1/2}H^{(r)}D_g^{-1/2}$ is Hermitian, hence $L^{(r)}$ is Hermitian.

For PSD, expand the quadratic form:
\[
\mathbf{x}^\ast L^{(r)}\mathbf{x}
=\mathbf{x}^\ast \mathbf{x}-\mathbf{x}^\ast D_g^{-1/2}H^{(r)}D_g^{-1/2}\mathbf{x}
=\sum_i |x_i|^2-\sum_{i,j}\overline{x_i}\frac{H^{(r)}_{ij}}{\sqrt{d_id_j}}x_j.
\]
Using $H^{(r)}_{ij}=W_{ij}e^{\mathrm{i}\Phi^{(r)}_{ij}}$ and $W_{ij}=W_{ji}$, symmetrize the second term:
\begin{equation*}
\begin{split}
&\sum_{i,j}\overline{x_i}\frac{H^{(r)}_{ij}}{\sqrt{d_id_j}}x_j \\
&\quad=\frac12\sum_{i,j}W_{ij}\left(
\overline{x_i}\frac{e^{\mathrm{i}\Phi^{(r)}_{ij}}}{\sqrt{d_id_j}}x_j
+\overline{x_j}\frac{e^{\mathrm{i}\Phi^{(r)}_{ji}}}{\sqrt{d_jd_i}}x_i
\right) \\
&\quad=\operatorname{Re}\!\left(\sum_{i,j}W_{ij}\overline{\frac{x_i}{\sqrt{d_i}}}e^{\mathrm{i}\Phi^{(r)}_{ij}}\frac{x_j}{\sqrt{d_j}}\!\right).
\end{split}
\end{equation*}
Then
\begin{equation*}
\begin{split}
\mathbf{x}^\ast L^{(r)}\mathbf{x}
&=\frac12\sum_{i,j}W_{ij}\left(
\left|\frac{x_i}{\sqrt{d_i}}\right|^2
+\left|\frac{x_j}{\sqrt{d_j}}\right|^2
-2\operatorname{Re}\!\left(\overline{\frac{x_i}{\sqrt{d_i}}}e^{\mathrm{i}\Phi^{(r)}_{ij}}\frac{x_j}{\sqrt{d_j}}\right)
\right) \\
&=\frac12\sum_{i,j}W_{ij}\left|
\frac{x_i}{\sqrt{d_i}}-e^{\mathrm{i}\Phi^{(r)}_{ij}}\frac{x_j}{\sqrt{d_j}}
\right|^2,
\end{split}
\end{equation*}
which is exactly~\eqref{eq:app_quadratic_form} and is nonnegative.
\end{proof}

\begin{corollary}[Real spectrum and orthonormal eigenbasis]
\label{cor:app_real_spectrum}
For each basis $r$, the eigenvalues of $L^{(r)}$ are real and nonnegative, and $L^{(r)}$ admits an orthonormal eigenvector basis in $\mathbb{C}^{|\mathcal{P}|}$.
\end{corollary}
\begin{proof}
A Hermitian matrix has a real spectrum and an orthonormal eigenbasis; positive semidefiniteness further implies nonnegative eigenvalues. Apply Lemma~\ref{lem:app_L_psd}.
\end{proof}

\subsection{Phase Tokens: Gauge Invariance and Information Preservation}
\label{app:gauge}

Let $V^{(r)}\in\mathbb{C}^{|\mathcal{P}|\times k}$ denote the $k$ eigenvectors of $L^{(r)}$ associated with the smallest $k$ eigenvalues. We use the smallest-$k$ eigenvectors (low-frequency modes) because they capture smooth, global directional structures aligned with the geographic graph, while higher-frequency components are more sensitive to noise and local sparsity in transition counts. And we define phase-only tokens
\begin{equation}
\label{eq:app_U_def}
U^{(r)}(i,m)=\exp(\mathrm{i}\,\arg(V^{(r)}(i,m)))\in\mathbb{C},
\qquad
|U^{(r)}(i,m)|=1.
\end{equation}
For a POI $p$, let $\mathbf{u}^{(r)}_p\in\mathbb{C}^{k}$ be the corresponding row of $U^{(r)}$.
The per-step phase-difference vector is
\begin{align*}
\Delta U^{(r)}_t &= \mathbf{u}^{(r)}_{p_t}\odot \overline{\mathbf{u}^{(r)}_{p^{\mathrm{src}}_t}}\in\mathbb{C}^{k} \\
\Delta U_t^{\operatorname{mix}} &= \sum_{r=1}^{R}\Pi_{b_t,r}\,\Delta U^{(r)}_t
\end{align*}

\begin{proposition}[Scale and global-phase invariance of phase differences]
\label{prop:app_gauge_invariance}
Fix a basis $r$ and a column index $m\in\{1,\ldots,k\}$.
Consider the column-wise gauge transformation of eigenvectors
\begin{equation}
\label{eq:app_gauge_transform}
V^{(r)}(:,m)\ \mapsto\ V^{(r)\prime}(:,m)=a_m e^{\mathrm{i}\phi_m}V^{(r)}(:,m),
\qquad a_m>0,\ \phi_m\in\mathbb{R}.
\end{equation}
Let $U^{(r)}$ and $U^{(r)\prime}$ be constructed by~\eqref{eq:app_U_def} from $V^{(r)}$ and $V^{(r)\prime}$, respectively.
Then for every POI pair $(p,q)$ and every $m$,
\begin{equation}
\label{eq:app_phase_diff_invariant}
U^{(r)\prime}(p,m)\,\overline{U^{(r)\prime}(q,m)}\ =\ U^{(r)}(p,m)\,\overline{U^{(r)}(q,m)}.
\end{equation}
Consequently, $\Delta U^{(r)}_t$ in \eqref{eq:DeltaU_r_t} is invariant to the eigenvector scale and global phase ambiguity in \eqref{eq:app_gauge_transform}.

\end{proposition}
\begin{proof}
For any nonzero complex number $z$, $\arg(ae^{\mathrm{i}\phi}z)=\arg(z)+\phi$ for $a>0$.
Applying~\eqref{eq:app_gauge_transform} entrywise yields
\begin{equation}
\begin{split}
U^{(r)\prime}(p,m)&=\exp(\mathrm{i}\arg(V^{(r)\prime}(p,m))) \\
&=\exp(\mathrm{i}(\arg(V^{(r)}(p,m))+\phi_m)) \\
&=e^{\mathrm{i}\phi_m}U^{(r)}(p,m).
\end{split}
\end{equation}
Therefore, for any $(p,q)$,
\begin{equation}
\begin{split}
&U^{(r)\prime}(p,m)\,\overline{U^{(r)\prime}(q,m)} \\
&\quad=\left(e^{\mathrm{i}\phi_m}U^{(r)}(p,m)\right)\overline{\left(e^{\mathrm{i}\phi_m}U^{(r)}(q,m)\right)} \\
&\quad=e^{\mathrm{i}\phi_m}U^{(r)}(p,m)\cdot e^{-\mathrm{i}\phi_m}\overline{U^{(r)}(q,m)} \\
&\quad=U^{(r)}(p,m)\,\overline{U^{(r)}(q,m)}.
\end{split}
\end{equation}
This proves~\eqref{eq:app_phase_diff_invariant}, and the elementwise form immediately implies invariance of $\Delta U^{(r)}_t$.
\end{proof}

\textbf{Scale Invariance of Static Geographic Information}

The static geographic information of the geographic graph—including distance-dependent relationships—is encoded by the symmetric weight matrix $W$ (as shown in \eqref{eq:geo_weight}), where $W_{ij} \geq 0$ and $W_{ij} = W_{ji}$ for all $i,j$. In the construction of the magnetic Laplacian, $W$ is embedded into the normalized Laplacian as
\begin{equation}
L^{(r)} = I_{|\mathcal{P}|} - D_g^{-1/2} H^{(r)} D_g^{-1/2},
\end{equation}
where $D_g=\mathrm{diag}(d_1,\ldots,d_{|\mathcal{P}|})$ is the diagonal degree matrix with $d_i = \sum_{j} W_{ij}$. 

A crucial property of this construction is its \textbf{scale invariance} with respect to the absolute magnitude of geographic weights. Specifically, for any global scaling factor $c > 0$, if we define a scaled weight matrix $W' = cW$, the corresponding degree matrix becomes $D_g' = cD_g$. The normalized terms then satisfy:
\begin{equation}
\left(D_g'^{-1/2} W' D_g'^{-1/2}\right)_{ij} = \frac{c W_{ij}}{\sqrt{(c d_i)(c d_j)}} = \frac{W_{ij}}{\sqrt{d_i d_j}} = \left(D_g^{-1/2} W D_g^{-1/2}\right)_{ij}.
\end{equation}
This property ensures that the spectral representation is invariant to the absolute units or scale of geographic distances (e.g., meters vs. kilometers), preserving the intrinsic relative spatial correlations encoded in the symmetric structure of $W$ without being distorted by global intensity shifts.

Furthermore, we retain only the pointwise phase of each eigenvector in $V^{(r)}$ to form
\begin{equation}
U^{(r)} = \exp\!\left(\mathrm{i}\arg(V^{(r)})\right),
\end{equation}
which, consistent with the global-phase invariance of abelian gauge transformations (Proposition \ref{prop:app_gauge_invariance}), is invariant under the global phase transformation $V^{(r)} \mapsto V^{(r)} e^{\mathrm{i}\theta}$ (for $\theta \in \mathbb{R}$). This mapping eliminates the column-wise unidentifiable global phase degree of freedom inherent in spectral decomposition. Consequently, leveraging the established gauge invariance, the transition from $V^{(r)}$ to $U^{(r)}$ ensures that the learned geographic bases are uniquely determined within their gauge equivalence classes.

\vspace{6pt}
\textbf{Direction-Sensitive Preservation of Dynamic Asymmetric Information}
The directional asymmetry in urban flows is encoded by the antisymmetric matrix $A^{(r)}$ (\eqref{eq:main_A_def}), whose antisymmetry is inherited by $\Phi^{(r)}$ (\eqref{eq:Phi_H_r}) and preserved through the spectral decomposition of $L^{(r)}$ (\eqref{eq:L_r}). The key direction-sensitive property lies in the phase difference construction:
\begin{equation}
\Delta U^{(r)}_t = \mathbf{u}^{(r)}_{p_t} \odot \overline{\mathbf{u}^{(r)}_{p_t^{\mathrm{src}}}},
\end{equation}
This operation is direction-sensitive for reversed pairs $(p_t, p_t^{\mathrm{src}}) \neq (p_t^{\mathrm{src}}, p_t)$—swapping the pair equivalently takes the complex conjugate, yielding distinguishable phase difference vectors. Thus, the antisymmetric directional information is faithfully preserved without aliasing, ensuring reliable mapping of asymmetric flow patterns to phase features.

\vspace{6pt}
\textbf{Faithful Real Encoding from Spectral Representation to $\mathbf{m}_t$}
\begin{equation}
\mathbf{m}_t= \operatorname{Concat}[\operatorname{Re}(\Delta U_t^{\operatorname{mix}}),\operatorname{Im}(\Delta U_t^{\operatorname{mix}})]\in\mathbb{R}^{2k},
\end{equation}
This concatenation of real and imaginary parts establishes a bijection between the complex representation $\Delta U_t^{\operatorname{mix}}$ and the real-valued vector $\mathbf{m}_t$, ensuring no loss of the encoded directional and spectral information.

\textbf{Summary on Model Design Rationality}
Rationality of the proposed phase token module lies in the fact that all key mappings from $W$ and $A^{(r)}$ to $\mathbf{m}_t$ are direction-sensitive and information-preserving (in the sense of corresponding equivalence classes), ensuring effective preservation of core static geographic distance information and dynamic asymmetric flow information throughout the whole process. This reliable information retention guarantees that the module provides a reliable and consistent geometric prior for the subsequent Mag-Mamba layer, thus verifying the rationality of our model design.

\begin{remark}[Degenerate eigenspaces]
\label{rem:app_degenerate}
If an eigenvalue of $L^{(r)}$ has multiplicity $>1$, the associated eigenvectors are defined up to a unitary rotation within the eigenspace, which is more general than the column-wise gauge in \eqref{eq:app_gauge_transform}.
In practice, exact degeneracy is measure-zero (i.e., occurs with probability zero for generic weights) under generic weighted graphs; furthermore, deterministic eigensolvers with fixed ordering and the phase-difference construction substantially reduce sensitivity to residual numerical rotations.
\end{remark}

\subsection{Offline and Online Complexity Analysis}
\label{app:complexity}

We separate costs into \emph{offline precomputation} (performed once on the training set) and \emph{online inference} (performed per sequence).
Let sequence length $L$, hidden size $D$, number of phase bases $R$, spectral truncation $k$, and number of time bins $N_b$.
Denote by $\mathrm{nnz}=O(|\mathcal{E}|)$ the number of nonzeros in sparse graph matrices.

\begin{lemma}[Offline cost is one-time precomputation]
\label{lem:app_offline_cost}
The spectral decomposition required to compute $\{U^{(r)}\}_{r=1}^R$ is a one-time offline cost and does not appear in online inference latency.
Concretely, the offline pipeline consists of:
(i) building $S\in\mathbb{R}^{N_b\times |\mathcal{E}|}$ from training transitions;
(ii) computing a rank-$R$ factorization $S\approx \Pi\Psi$;
(iii) for each basis $r$, forming the sparse Hermitian $L^{(r)}$ and computing the $k$ smallest eigenpairs to obtain $V^{(r)}$ and hence $U^{(r)}$.
All these steps are independent of the test sequence length $L$ and can be precomputed and cached once.
\end{lemma}
\begin{proof}
Step (i): building $S$ requires counting directed transitions per time bin and per undirected edge.
Let $T$ be the number of transitions in the training set.
With hashing from ordered POI pairs to edge indices, the counting is $O(T)$ time and $O(N_b|\mathcal{E}|)$ storage for dense $S$ (or $O(T)$ if stored sparsely and then densified).

Step (ii): computing a rank-$R$ factorization of $S\in\mathbb{R}^{N_b\times |\mathcal{E}|}$ can be done by truncated SVD / randomized low-rank methods with time
\[
O(N_b |\mathcal{E}| R)
\]
(up to logarithmic and oversampling factors) and memory $O((N_b+|\mathcal{E}|)R)$.
This produces $\Pi\in\mathbb{R}^{N_b\times R}$ and $\Psi\in\mathbb{R}^{R\times |\mathcal{E}|}$.

Step (iii): for each $r$, forming $A^{(r)}$ and $H^{(r)}$ is $O(|\mathcal{E}|)$, and forming the sparse normalized Laplacian $L^{(r)}$ is $O(|\mathcal{E}|)$.
Computing the $k$ smallest eigenpairs of a sparse Hermitian matrix can be done by Lanczos/ARPACK-type solvers in
\[
O\bigl(I_r\cdot k\cdot \mathrm{nnz}\bigr)=O\bigl(I_r\cdot k\cdot |\mathcal{E}|\bigr)
\]
where $I_r$ is the iteration count (data-dependent but typically modest for well-conditioned sparse graphs).
This avoids dense $O(|\mathcal{P}|^3)$ decomposition; only a truncated spectrum of size $k\ll |\mathcal{P}|$ is required.

Crucially, all three steps depend only on the training trajectories and the fixed geographic graph.
Once $\Pi$ and $\{U^{(r)}\}$ are computed, online inference only performs table lookups and elementwise operations; no eigen-decomposition is executed at inference time.
Therefore, the spectral computation is an offline one-time precomputation and does not contribute to per-sequence inference latency.
\end{proof}

\begin{lemma}[Online linearity of Mag-Mamba inference]
\label{lem:app_online_linearity}
Assume $R$ and $k$ are fixed hyperparameters independent of $L$ and $|\mathcal{P}|$.
Then the end-to-end inference complexity of the Mag-Mamba layer over a length-$L$ sequence is
\begin{equation}
\label{eq:app_online_complexity}
O(L\cdot D),
\end{equation}
i.e., strictly linear in the sequence length.
Moreover, the additional overhead introduced by magnetic phases is $O(L\cdot (Rk + Dk))=O(L\cdot D)$ under fixed $k$.
\end{lemma}
\begin{proof}
We decompose the online computation per step $t$:

\noindent\textbf{(1) Phase-difference feature $\mathbf{m}_t$.}
For each basis $r$, $\Delta U^{(r)}_t$ is computed by
\[
\Delta U^{(r)}_t=\mathbf{u}^{(r)}_{p_t}\odot \overline{\mathbf{u}^{(r)}_{p^{\mathrm{src}}_t}}\in\mathbb{C}^{k}
\]
which requires two row lookups and $k$ elementwise complex multiplications: $O(k)$ time.
Mixing across bases
\[
\Delta U_t^{\operatorname{mix}}=\sum_{r=1}^{R}\Pi_{b_t,r}\Delta U^{(r)}_t
\]
costs $O(Rk)$.
Forming $\mathbf{m}_t=\mathrm{Concat}[\operatorname{Re}(\Delta U_t^{\operatorname{mix}}),\operatorname{Im}(\Delta U_t^{\operatorname{mix}})]\in\mathbb{R}^{2k}$ costs $O(k)$.
Thus $\mathbf{m}_t$ costs $O(Rk)$ per step.

\noindent\textbf{(2) Angular velocity and rotation angles.}
Computing $\boldsymbol{\theta}^{\mathrm{mag}}_t=W_{\theta m}\mathbf{m}_t$ costs $O(Dk)$ since $W_{\theta m}\in\mathbb{R}^{(D/2)\times 2k}$.
Computing $\boldsymbol{\theta}^{\mathrm{tok}}_t=W_{\theta x}\mathbf{u}_t$ is the same token-side cost as in Mamba-3 and is linear in $D$ under the standard per-channel parameterization used by Mamba-style SSM blocks.
Computing $\boldsymbol{\varphi}_t=\boldsymbol{\delta}^{\mathrm{half}}_t\odot \boldsymbol{\theta}_t$ is elementwise: $O(D/2)$.

\noindent\textbf{(3) State update.}
The recurrence~\eqref{eq:h_t} applies:
(i) one block rotation $R_t(\mathbf{h}_{t-1})$ over $D/2$ pairs, which costs $O(D)$;
(ii) elementwise multiplications by $\boldsymbol{\alpha}_t,\boldsymbol{\beta}_t,\boldsymbol{\gamma}_t$ and additions, all $O(D)$.
This matches the Mamba-3 per-step cost.

Aggregating, the magnetic part adds $O(Rk + Dk)$ per token.
With fixed hyperparameters $R,k$, $Rk=O(1)$ and $Dk=O(D)$, hence the per-step cost remains $O(D)$ and the total cost over $L$ steps is $O(LD)$.
No step involves graph-size-dependent operations (no $|\mathcal{P}|$-, $|\mathcal{E}|$-, or $|\mathcal{P}|^3$-scale routines) because $\{U^{(r)}\}$ and $\Pi$ are cached lookup tables.
\end{proof}

\subsection{Control Dynamics of Phase Injection}
\label{app:control_dynamics}

This section formalizes the role of the magnetic phase feature $\mathbf{m}_t$ as an additive forcing term on the system's rotational dynamics.
Define the angular velocity decomposition (restating the model definition)
\begin{align*}
\boldsymbol{\theta}_t &= \boldsymbol{\theta}^{\mathrm{tok}}_t+\boldsymbol{\theta}^{\mathrm{mag}}_t \\
\boldsymbol{\theta}^{\mathrm{mag}}_t &= W_{\theta m}\mathbf{m}_t
\end{align*}
Let $\boldsymbol{\varphi}_t={\boldsymbol{\delta}^{\mathrm{half}}_t}\odot \boldsymbol{\theta}_t$ denote the per-pair rotation angles.

\begin{proposition}[Phase injection as topological forcing on instantaneous frequency]
\label{prop:app_topological_forcing}
Represent the $d$-th interleaved 2D state pair as a complex scalar
\begin{equation}
\label{eq:app_complex_state}
z_{t,d}=h_{t,2d-1}+\mathrm{i}h_{t,2d}\in\mathbb{C},\qquad d=1,\ldots,D/2.
\end{equation}
Then the rotation operator $R_t(\cdot)$ in~\eqref{eq:h_t} acts as complex multiplication:
\begin{equation}
\label{eq:app_rotation_complex}
\bigl[R_t(\mathbf{h}_{t-1})\bigr]_{(d)} \ \Longleftrightarrow\ e^{\mathrm{i}\varphi_{t,d}}\,z_{t-1,d},
\end{equation}
where $[\cdot]_{(d)}$ denotes the $d$-th complex pair under~\eqref{eq:app_complex_state}.
Consequently, the homogeneous part of the state evolution (ignoring input injections) is
\begin{equation}
\label{eq:app_homogeneous}
z_{t,d}=\alpha_{t,(d)}\,e^{\mathrm{i}\varphi_{t,d}}\,z_{t-1,d},
\end{equation}
and the cumulative phase advance over time is
\begin{equation}
\label{eq:app_phase_accumulation}
\begin{split}
\arg(z_{t,d})-\arg(z_{0,d})&=\sum_{s=1}^{t}\varphi_{s,d}\ \ (\mathrm{mod}\ 2\pi), \\
\varphi_{s,d} = \delta^{\mathrm{half}}_{s,d}\Bigl(\theta^{\mathrm{tok}}_{s,d}+[W_{\theta m}\mathbf{m}_s]_d\Bigr).
\end{split}
\end{equation}
Define the instantaneous (discrete) angular frequency as
\begin{equation}
\label{eq:app_instant_freq}
\omega_{t,d}\ \triangleq\ \frac{\varphi_{t,d}}{\delta^{\mathrm{half}}_{t,d}}
\ =\ \theta^{\mathrm{tok}}_{t,d}+[W_{\theta m}\mathbf{m}_t]_d.
\end{equation}
Then $\mathbf{m}_t$ contributes an \emph{additive forcing term} $[W_{\theta m}\mathbf{m}_t]_d$ to the system's natural token-driven frequency $\theta^{\mathrm{tok}}_{t,d}$.
Equivalently, compared to a baseline without magnetic injection, the phase trajectory is shifted by
\begin{equation}
\label{eq:app_forcing_shift}
\begin{split}
\Delta \Bigl(\arg(z_{t,d})\Bigr)
&=\sum_{s=1}^t \delta^{\mathrm{half}}_{s,d}\,[W_{\theta m}\mathbf{m}_s]_d \\
&\quad (\mathrm{mod}\ 2\pi).
\end{split}
\end{equation}
\end{proposition}

\begin{proof}
By definition, $R_t(\cdot)$ rotates each 2D pair by angle $\varphi_{t,d}$.
Under the identification~\eqref{eq:app_complex_state}, a 2D rotation is exactly multiplication by $e^{\mathrm{i}\varphi_{t,d}}$, proving~\eqref{eq:app_rotation_complex}.
Substituting into the first term of~\eqref{eq:h_t} yields the homogeneous evolution~\eqref{eq:app_homogeneous}.
Iterating~\eqref{eq:app_homogeneous} gives
\begin{equation}
z_{t,d}=\left(\prod_{s=1}^t \alpha_{s,(d)}\right)\exp\!\left(\mathrm{i}\sum_{s=1}^t\varphi_{s,d}\right)z_{0,d},
\end{equation}
which implies the phase accumulation~\eqref{eq:app_phase_accumulation}.
Using~\eqref{eq:app_phi_def} and~\eqref{eq:theta_t_all}, we have
\begin{equation}
\begin{split}
\varphi_{t,d} &=\delta^{\mathrm{half}}_{t,d}\bigl(\theta^{\mathrm{tok}}_{t,d}+\theta^{\mathrm{mag}}_{t,d}\bigr) \\
&=\delta^{\mathrm{half}}_{t,d}\Bigl(\theta^{\mathrm{tok}}_{t,d}+[W_{\theta m}\mathbf{m}_t]_d\Bigr),
\end{split}
\end{equation}
which directly yields the instantaneous frequency definition~\eqref{eq:app_instant_freq}.
Finally, subtracting the phase accumulation of a baseline model with $\mathbf{m}_t\equiv \mathbf{0}$ gives the shift~\eqref{eq:app_forcing_shift}.
\end{proof}

\begin{remark}[Interpretation]
\label{rem:app_forcing_interp}
Equation~\eqref{eq:app_instant_freq} shows that Mag-Mamba injects a graph-induced phase control as an additive term on the rotation frequency.
Since $\mathbf{m}_t$ is derived from a Hermitian spectral construction on the geographic graph (Appendix~\ref{app:hermitian}), this term acts as a topology-conditioned forcing on the dynamical phase evolution.
\end{remark}

\end{document}